\documentclass{article}

\usepackage{microtype}
\usepackage{graphicx}
\usepackage{subfigure}
\usepackage{booktabs} 
\usepackage{xcolor}

\usepackage[hidelinks]{hyperref}
\definecolor{darkred}{RGB}{150,0,0}
\definecolor{darkgreen}{RGB}{0,150,0}
\definecolor{darkblue}{RGB}{0,0,150}
\hypersetup{colorlinks=true, linkcolor=red, citecolor=blue, urlcolor=darkblue}

\usepackage[accepted]{icml2022}

\usepackage{amsmath}
\usepackage{amssymb}
\usepackage{mathtools}
\usepackage{amsthm}

\usepackage[capitalize,noabbrev]{cleveref}

\theoremstyle{plain}
\newtheorem{theorem}{Theorem}[section]
\newtheorem{proposition}[theorem]{Proposition}
\newtheorem{lemma}[theorem]{Lemma}

\theoremstyle{definition}

\newtheorem{assumption}[theorem]{Assumption}
\newtheorem{remark}[theorem]{Remark}

\usepackage[textsize=tiny]{todonotes}
\usepackage{tabulary}
\usepackage{ textcomp }
\usepackage{pdfpages}
\usepackage{colortbl}
\usepackage{lipsum}
\usepackage{bbm}

\newcommand{\mR}{\mathbb{R}}

\newcommand{\norm}[1]{\left\lVert#1\right\rVert}

\DeclareMathOperator*{\argmin}{argmin}
\DeclareMathOperator*{\argmax}{argmax}

\icmltitlerunning{Feature  and Parameter Selection in Stochastic Linear Bandits}

\begin{document}

\twocolumn[
\icmltitle{ Feature  and Parameter Selection in Stochastic Linear Bandits}




\begin{icmlauthorlist}
\icmlauthor{Ahmadreza Moradipari}{yyy}
\icmlauthor{Berkay Turan}{yyy}
\icmlauthor{Yasin Abbasi-Yadkori}{comp}
\icmlauthor{Mahnoosh Alizadeh}{yyy}
\icmlauthor{Mohammad Ghavamzadeh}{sch}
\end{icmlauthorlist}

\icmlaffiliation{yyy}{Department of Electrical and Computer Engineering, University of California, Santa Barbara, USA}
\icmlaffiliation{comp}{DeepMind, London, UK}
\icmlaffiliation{sch}{Google Research, Mountain View, USA}

\icmlcorrespondingauthor{Ahamdreza Moradipari}{ahmadreza$\_$moradipari@ucsb.edu}


\icmlkeywords{Machine Learning, ICML}

\vskip 0.3in
]



\printAffiliationsAndNotice{}  

\begin{abstract}
We study two model selection settings in stochastic linear bandits (LB). In the first setting, which we refer to as {\em feature selection}, the expected reward of the LB problem is in the linear span of at least one of $M$ feature maps (models). In the second setting, the reward parameter of the LB problem is arbitrarily selected from $M$ models represented as (possibly) overlapping balls in $\mathbb R^d$. However, the agent only has access to misspecified models, i.e.,~estimates of the centers and radii of the balls. We refer to this setting as {\em parameter selection}. For each setting, we develop and analyze a computationally efficient algorithm that is based on a reduction from bandits to full-information problems. This allows us to obtain regret bounds that are not worse (up to a $\sqrt{\log M}$ factor) than the case where the true model is known. This is the best reported dependence on the number of models $M$ in these settings. Finally, we empirically show the effectiveness of our algorithms using synthetic and real-world experiments.
\end{abstract}

\section{Introduction}
\label{sec:intro}

Learning under bandit feedback is a class of online learning problems in which an agent interacts with the environment through a set of actions (arms), and receives rewards only from the arms that it has pulled. The goal of the agent is to maximize its expected cumulative reward without knowledge of the reward distributions of the arms. {\em Multi-armed bandit} (MAB) is the simplest form of this problem~\citep{lai85asymptotically,Auer,LS-2020, moradipari2018learning}. {\em Linear bandit}~\citep{Dani08stochasticlinear,rusmevichientong2010linearly,abbasi2011improved} is a generalization of MAB to (possibly) infinitely many arms, each associated with a feature vector. The mean reward of each arm is assumed to be the dot product of its feature vector and an {\em unknown} parameter vector. This setting contains {\em contextual linear bandit} in which action sets and feature vectors change at every round.
The main component of bandit algorithms is to balance {\em exploration} and {\em exploitation}: to decide when to {\em explore} and learn about the arms, and when to {\em exploit} and select the action with the highest estimated reward. The most common exploration strategies are {\em optimism in the face of uncertainty} (OFU) or upper confidence bound (UCB)~\citep{Auer,Dani08stochasticlinear,abbasi2011improved,moradipari2020stage,moradipari2022collaborative}, and Thompson sampling (TS)~\citep{thompson1933likelihood,agrawal2013thompson,russo,abeille2017linear,moradipari2020linear,moradipari2021safe}.  

In this paper, we study {\em model selection} in stochastic linear bandits (LB), where the LB problem at hand is selected from a set of $M$ models. The agent has information about the models but does not know the identity of the one(s) that the new LB problem has been selected from. The goal of the agent is to identify the true model(s) and transfer its (their) collected experience to speedup the learning of the task at hand. It is a common scenario in many application domains that the new task belongs to a family of models that are either known accurately or with misspecification. For example, it is reasonable to assume that the customers of an online marketing website, the users of an app, or the patients in a medical trial belong to a certain number of categories based on their shopping and browsing habits or their genetic signatures. It is also common these days that websites, apps, and clinics have a large amount of information from each of these categories that can be used to build a model. 

Model selection is particularly challenging with bandit information. A common approach is to consider each model as a black-box that runs a bandit algorithm with its own information, and then a meta algorithm plays a form of bandit-over-bandits strategy with their outcomes. These algorithms often achieve a regret of $\widetilde{\mathcal O}(\sqrt{MT})$, and thus, are not desirable when the number of models $M$ is large. In this paper, we consider two bandit model selection settings and show that it is possible to improve this rate so that the regret scales as $\sqrt{\log M}$ with the number of models. The main innovation in our algorithms is utilizing reductions from bandits to full-information problems, and performing model selection in the full-information setting for which much stronger results exist. The main reason for $\widetilde{\mathcal O}(\sqrt{MT})$ regret in bandit-over-bandits algorithms is that no information is shared among the models (bandit algorithms), i.e.,~when a bandit algorithm is used to take an action in a round, the resulting feedback is not shared with the other models. On the other hand, model selection in the full-information setting allows the model to share information among each other, which makes the superior $\sqrt{\log M}$ regret bound possible. 

The two model selection settings we consider in this paper are: {\em feature selection}, where the mean reward of the LB problem is in the linear span of at least one of $M$ given feature maps (models), and {\em parameter selection}, where the reward parameter of the LB problem is arbitrarily selected from $M$ models represented as (possibly) overlapping balls in $\mathbb R^d$. Here the models can be misspecified, i.e.,~only estimates of the centers and radii of the balls are given to the algorithm. We derive algorithms in these settings that use reductions from bandits to full-information. Our algorithms are computationally efficient and have regret bounds that are not worse (up to a $\sqrt{\log M}$ factor) than the case where the true model is known. We achieve this by properly instantiating existing algorithmic paradigms: SquareCB \citep{FR-2020} and OFUL \citep{abbasi2011improved}. 
The SquareCB algorithm in its original form uses a set of static experts, but we need adaptive (learning) experts in order to have a computational efficient algorithm with the desired regret in our {\em feature selection} setting. Working with adaptive (time-varying) experts requires appropriate and non-trivial modifications to the proof of SquareCB.

There are mainly two types of reductions from bandits to full-information problems. The first one is the classical reduction that uses importance weighted estimates. A popular algorithm in this class is EXP3 that uses Exponentially Weighted Average forecaster as the full-information algorithm. The bandit model selection strategy of~\citet{ALNS-2017}, known as CORRAL, also uses this type of reduction with an online mirror descent method and a carefully selected mirror map as the full-information algorithm. Given that importance weighted estimates are fed to the full-information algorithm, a $\sqrt{M}$ term is in general unavoidable in the regret of the methods that use this type of reduction. In this work, we use a different type of full-information reduction introduced by~\citet{FR-2020} and~\citet{APS-2012}. Here, the full-information algorithm has direct access to its losses without any importance weighted estimates, and thus, allows us to obtain regrets that scales as $\sqrt{\log M}$.


\section{Problem Formulation}
\label{sec:prob-formulation}

In this section, we first provide a brief overview of stochastic linear bandits. We then describe the two model selection settings studied in the paper. We conclude by introducing a regression oracle used by our algorithms that is based on sequential prediction with expert advice and square loss. 


\subsection{Stochastic Linear Bandits}
\label{subsec:LB}

A stochastic linear bandit (LB) problem is defined by a sequence of $T$ interactions of a learning agent with a stochastic environment. At each round $t\in [T]$, the agent is given a decision set $\mathcal A_t\subset \mathbb R^d$ from which it has to select an action $a_t$. Upon taking the action $a_t\in\mathcal A_t$, it observes a reward $y_t=\langle \phi_t(a_t),\theta_*\rangle + \eta_t$, where $\theta_*\in\mathbb R^d$ is the unknown reward parameter, $\phi_t(a)\in\mathbb R^d$ is the feature vector of action $a$ at round $t$, and $\eta_t$ is a zero-mean $R$-sub-Gaussian noise. When the features correspond to the canonical basis, this formulation reduces to {\em multi-armed bandit}. In case the features depend on both an action $a\in\mathcal A$ and a context $x\in\mathcal X$, i.e.,~$\phi_t(a_t)=\phi(x_t,a_t)$, this LB formulation is called {\em contextual linear bandit}. It is also common in practice that the action set is fixed and finite, i.e.,~$\mathcal A = [K]$, in which case we are in the finite $K$-action setting. The history $H_t$ of a LB algorithm up to round $t$ consists of all the contexts, actions, and rewards that it has observed from the beginning until the end of round $t-1$, i.e.,~$H_t=\{(x_s,a_s,y_s)\}_{s=1}^{t-1}$, or equivalently $H_t=\{(\phi_s(a_s),y_s)\}_{s=1}^{t-1}$.

The goal of the agent in LB is to maximize its expected cumulative reward in $T$ rounds, or equivalently to minimize its $T$-round (pseudo) regret, i.e.,

\vspace{-0.15in}
\begin{small}
\begin{equation}
\label{eq:regret}
\mathcal R(T,\theta_*) = \sum_{t=1}^T \langle \phi_t(a_t^*),\theta_* \rangle - \langle \phi_t(a_t),\theta_* \rangle,    
\end{equation}
\end{small}
\vspace{-0.15in}

where $a^*_t = \argmax_{a\in\mathcal A_t}\langle \phi_t(a),\theta_*\rangle$ is the optimal action in round $t$.


\subsection{Feature Selection Setting} 
\label{subsec:feature-selection}

In this setting, the agent is given a set of $M$ feature maps $\{\phi^i\}_{i=1}^M$ with dimension $d$. We assume that the expected reward of the LB problem belongs to the linear span of at least one of these $M$ models (features), i.e.,~there exists an $i\in [M]$ and a $\theta_*^i\in\mathbb R^d$, such that for all rounds $t\in [T]$, contexts $x\in\mathcal X$, and actions $a\in\mathcal A$, we may write the mean rewards as $\mathbb E[y_t] = \langle\phi^i(x,a),\theta_*^i\rangle$.\footnote{Note that we use the contextual linear bandit notation for this setting and in the corresponding sections.} 
We refer to such feature maps as {\em true} models and denote them by $i_*$.
Note that the agent does not know the identity of the true model(s) $i_*$. 

As a motivational example for this setting, we can consider a recommender system that has trained $M$ models (e.g.,~$M$ neural networks) to predict the score of customer-item pairs. Each model corresponds to a particular mood or type of the customer, or any other latent component of the customer's state. Each model provides an embedding for customer-item pairs and the score is linear in this embedding (think of an embedding as the one to the last layer of a trained NN). When a new customer arrives, the recommender system should find out as soon as possible which of the $M$ models (embeddings) is the best match to the current mood/type of this customer in order to recommend her desirable items.

We make the following standard assumption on the boundedness of the reward parameters and features of the $M$ models. 
\begin{assumption}
\label{ass:boundedness-Setting2}
There are constants $L,S,G \geq 0$, such that for all $i\in[M]$, $t\in [T]$, $x\in\mathcal X$, and $a\in\mathcal A$, we have $\|\theta_*^i\|\leq S$, $\|\phi^i(x,a)\|\leq L$, and $|\langle \phi^i(x,a),\theta_*^i \rangle| \leq G$.
\end{assumption}
%
%
Our goal here is to design an algorithm that minimizes {\em transfer regret}, which in this setting we define it as

\vspace{-0.15in}
 \begin{small}
\begin{equation}
\label{eq:transfer-regret2}
\mathcal R(T) = \sum_{t=1}^T \langle \phi^{i_*}(x_t,a_t^*),\theta_*^{i_*} \rangle - \langle \phi^{i_*}(x_t,a_t),\theta_*^{i_*} \rangle, 
\end{equation}
\end{small}
\vspace{-0.15in}

where $a_t^* = \argmax_{a\in\mathcal A} \langle \phi^{i_*}(x_t,a),\theta_*^{i_*} \rangle$. In the results we report for this setting in Section~\ref{sec:feature-select-alg}, we make two assumptions: {\bf 1)} the feature maps are all known (no model misspecification), and {\bf 2)} the number of actions is finite, i.e., we are in the finite $K$-action setting described in Section~\ref{subsec:LB}. However, we believe that our algorithm and analysis can be extended to the case of having misspecified models and convex action sets using the results in~\citet{Foster21AM}. 


\subsection{Parameter Selection Setting} 
\label{subsec:param-selection}

In this setting, unlike the classical setting in Section~\ref{subsec:LB}, we no longer assume that the unknown parameter $\theta_*$ can be any vector in $\mathbb{R}^d$. Rather, $\theta_*$ can be generated from $M$ possible reward models, each defined as a ball $B(\mu_i,b_i) = \{\theta\in\mathbb R^d:\|\theta - \mu_i\| \leq b_i\}$, with center $\mu_i\in\mathbb R^d$ and radius $b_i\geq 0$. Note that the models (balls) may overlap and do not have to be disjoint. The $M$ models can be thought of the responses of $M$ types (or clusters) of customers to different items in a recommender system or the reactions of patients with $M$ genotypes to a set of drugs. The radii $\{b_i\}_{i=1}^M$ represent the variation within each cluster. The reward parameter $\theta_*$ of the new task (LB problem) is arbitrarily selected from the $M$ models. For example, it can be adversarially selected from the union of the models, i.e.,~$\theta_*\in\bigcup_{i=1}^M B(\mu_i,b_i)$. In this case, we denote by $\mathcal I_*$, the set of indices of the balls that contain $\theta_*$. 
Since the models are often computed from (finite) historical data, it is reasonable to assume that only {\em estimates} of their centers $\{\widehat\mu_i\}_{i=1}^M$ are available, together with upper-bounds on the error of these estimates $\{c_i\}_{i=1}^M$, such that $\|\mu_i - \widehat\mu_i\| \leq c_i$, for all $i\in [M]$. 

The agent has no knowledge either about $\theta_*$ or the process according to which it has been selected. The only information given to the agent are: {\bf 1)} estimates $\widehat\mu_i$ of the center of the models, {\bf 2)} upper-bounds $c_i$ on the errors of these estimates, and {\bf 3)} the exact radii $b_i$ of the models, for all $i\in[M]$. This means that although $\theta_*$ is selected from the {\em actual} models $B(\mu_i,b_i)$, the agent has only access to {\em estimated} models $B(\widehat\mu_i,b_i+c_i)$ that have more uncertainty (their corresponding balls are larger). For simplicity, we assume that the exact values of radii $\{b_i\}_{i=1}^M$ are known. However, our results can be easily extended to the case that instead of $b_i$'s, their estimates $\widehat b_i$ and upper-bounds on their errors $c'_i$, i.e.,~$\|b_i - \widehat b_i\| \leq c'_i$, for all $i \in [M]$, are given to the agent. In this case, the agent has to use even more uncertain estimates of the models $B(\widehat \mu_i,b_i+c_i+c'_i)$. 

Our goal is to design an algorithm that can transfer knowledge from these estimated models and learn the new task with parameter $\theta_*$ more efficiently than when it is independently learned. This goal can be quantitatively stated as minimizing the {\em transfer regret}, 

\vspace{-0.15in}
\begin{small}
\begin{equation}
\label{eq:transfer-regret}
\mathcal{R}(T) = \sup_{\theta_* \in \bigcup_{i=1}^M B(\mu_i, b_i)} \mathcal R(T,\theta_*), 
\end{equation} 
\end{small}
\vspace{-0.15in}
 
where $\mathcal R(T,\theta_*)$ is the regret defined by~\eqref{eq:regret}. We make the following standard assumption on the boundedness of the features and expected rewards. 
\begin{assumption}
\label{ass:boundedness-Setting1}
There exist constants $L,G \geq 0$, such that $\forall t\in [T]$ and $\forall a\in\bigcup_{t=1}^T\mathcal A_t$, we have $\|\phi_t(a)\|\leq L$, and $\forall \theta\in\bigcup_{i=1}^MB(\mu_i,b_i)$, we have $|\langle\phi_t(a),\theta\rangle|\leq G$.
\end{assumption}


\subsection{Regression Oracle}
\label{subseq:reg-oracle}

In both model selection settings studied in the paper, our proposed algorithms use a regression oracle that is based on sequential prediction with expert advice and square loss. Following~\citet{FR-2020} and~\citet{Foster21AM}, we refer to this regression oracle as \texttt{SqAlg}. We can consider \texttt{SqAlg} as a meta algorithm that consists of $M$ learning algorithms (or experts), each corresponding to one of our $M$ models, and returns a prediction by aggregating the predictions of its experts. More precisely, in each round $t\in[T]$, \texttt{SqAlg} takes the current context-action pair $(x_t,a_t)$, or equivalently $\phi_t(a_t)$, as input, and gives them to its $M$ experts to predict their reward, i.e.,~$f^i_t(H_t)=f^i(\phi_t(a_t);H_t),\;\forall i\in [M]$, given the current history $H_t$. Then, the meta algorithm \texttt{SqAlg} aggregates its experts' predictions, $\{f^i_t(H_t)\}_{i=1}^M$, given their current weights, and returns its own prediction $\widehat{y}_t=\texttt{SqAlg}_t(\phi_t(a_t);H_t)$. Upon observing the actual reward $y_t$, \texttt{SqAlg} updates the weights of its experts according to the difference between their predictions $f^i(\phi_t(a_t);H_t)$ and the actual reward $y_t$. 

The regression oracles (\texttt{SqAlg}) used by our model selection algorithms differ in the prediction algorithm used by their experts. However, in both cases, \texttt{SqAlg} aggregates its experts' predictions using an algorithm by~\citet{Herbster-Kivinen-Warmuth-1998} (see Algorithm~\ref{alg:expert-prediction} in Appendix~\ref{app:SqAlg-description}). The performance of \texttt{SqAlg} is evaluated in terms of its regret $\mathcal R_{\texttt{Sq}}(T)$, which is defined as its accuracy (in terms of square loss) w.r.t.~the accuracy of the best expert in the set, i.e.,

\vspace{-0.15in}
\begin{small}
\begin{equation}
\label{eq:regression-oracle}
\sum_{t=1}^T(\widehat{y}_t - y_t)^2 - \min_{i\in [M]}\sum_{t=1}^T(f^i_t(H_t) - y_t)^2 \leq {\mathcal R}_{\texttt{Sq}}(T).
\end{equation}
\end{small}
\vspace{-0.15in}

In each round $t$, we define the oracle prediction for a context $x$ and an action $a$ as $\widehat{y}_t(x,a):=\texttt{SqAlg}_t(x,a;H_t)$. As shown in~\citet{Herbster-Kivinen-Warmuth-1998}, in case all observations and experts' predictions are bounded in an interval of size $\ell$, this regret can be bounded as $\mathcal R_{\texttt{Sq}}(T)\leq \ell^2\log M$ (see Appendix~\ref{app:SqAlg-description} for more details). We use this regret bound in the analysis of our proposed algorithms. 


\section{Feature Selection Algorithm}
\label{sec:feature-select-alg}

In this section, we derive an algorithm for the feature selection setting described in Section~\ref{subsec:feature-selection} that is based on the SquareCB algorithm~\cite{FR-2020}. We refer to our algorithm as {\em feature selection SquareCB} (FS-SCB). We prove an upper-bound on the transfer regret of FS-SCB in Section~\ref{subsec:analysis-FS-SCB}, and provide an overview of the related work and a discussion on our results in Section~\ref{subsec:related-work-FS}.

Algorithm~\ref{alg:FS-SCB} contains the pseudo-code of FS-SCB. In each round $t\in[T]$, the algorithm observes a context $x_t\in\mathcal X$ and passes it to its regression oracle \texttt{SqAlg} to produce its reward predictions $\widehat{y}_t(x_t,a),\forall a\in[K]$. Each expert in \texttt{SqAlg} corresponds to one of the $M$ models and is a {\em ridge regression} algorithm with the feature map of that model. Expert $i\in[M]$ predicts the reward of the context $x_t$, for each action $a\in[K]$, as $f^i(x_t,a;H_t) = \langle \phi^i(x_t,a),\widehat{\theta}_t^i\rangle$, where $\widehat{\theta}_t^i = \argmin_{\theta} \|\Phi_t^{i\top}\theta - Y_t\|^2 + \lambda_i\|\theta\|^2$. We may write $\widehat{\theta}_t^i$ in closed-form as $\widehat{\theta}_t^i=(V_t^{\lambda_i})^{-1}\Phi_t^{i\top}Y_t$. In these equations, $Y_t=(y_1,\ldots,y_{t-1})^\top$ is the reward vector; $\Phi_t^i$ is the feature matrix of the $i^{th}$ model, whose rows are $\phi^i(x_1,a_1),\ldots,\phi^i(x_{t-1},a_{t-1})$; $\lambda_i$ is the regularization parameter of model $i$, which our analysis shows that it only needs to be larger than one, i.e.,~$\lambda_i \geq 1$; and finally $V_t^{\lambda_i} = \lambda_iI+\Phi_t^\top\Phi_t$. The meta algorithm \texttt{SqAlg} aggregates the experts' predictions $\{f^i(x_t,a;H_t)\}_{i=1}^M$ and produces its own predictions $\widehat{y}_t(x_t,a),\;\forall a\in[K]$, using Algorithm~\ref{alg:expert-prediction} in Appendix~\ref{app:SqAlg-description} (see Remark~\ref{admissible-expert-feature}).

\begin{algorithm}[tb]
  \caption{Feature Selection Square-CB (FS-SCB)}
  \label{alg:FS-SCB}
\begin{algorithmic}
  \STATE {\bfseries Input:} Models $\{\phi^i\}_{i=1}^M$, Confidence Parameter $\delta$, Learning Rate $\alpha$, Exploration Parameter $\kappa$
  \FOR{$t=1$ {\bfseries to} $T$}
  \STATE Observe context $x_{t}$ 
  \STATE Oracle predicts: \\$ \;\; \widehat{y}_t(x_t,a) = \texttt{SqAlg}_t(x_t,a;H_t), \;\; \forall a \in [K]$ \label{alg:FS-SCB:Prediction}
  \STATE Define a distribution $p_t$ over the actions:
\begin{equation}
\label{eq:dist-over-actions}
\hspace{-0.1in} p_t(a) = 
\begin{cases} 
\frac{1}{\kappa + \alpha \big(\widehat{y}_t(x_t,a) - \widehat{y}_t(x_t,a'_t)\big)}, & a \neq a'_t, \\ 
1 - \sum_{a \neq a'_t} p_t(a), & a = a'_t,
\end{cases}
\end{equation} 
where $a'_t = \argmax_{a\in[K]} \widehat{y}_t(x_t,a)$; \label{alg:FS-SCB:Action-Prob} 
\STATE Sample action $a_t \sim p_t(\cdot)$ and play it; \label{alg:FS-SCB:Action-Selection}
\STATE Observe reward $\;y_t = \langle \phi^{i_*}(x_t,a_t), \theta_*^{i_*} \rangle + \eta_t$; 
\STATE Update \texttt{SqAlg} with $(x_t,a_t,y_t)$; \label{alg:FS-SCB:Reward-Update}
  \ENDFOR
\end{algorithmic}
\end{algorithm}

The next step in FS-SCB is computing the action with the highest predicted reward, i.e.,~$a'_t=\argmax_{a\in[K]}\widehat{y}_t(x_t,a)$, and using it to define a distribution $p_t\in\Delta_K$ over the actions (see Eq.~\ref{eq:dist-over-actions}). The distribution $p_t$ in~\eqref{eq:dist-over-actions} is defined similarly to the probability selection scheme of~\citet{Abe99AR}, and assigns a probability to every action inversely proportional to the gap between its prediction and that of $a'_t$. The algorithm then samples its action $a_t$ from $p_t$, observes reward $y_t$, and feeds the tuple $(x_t,a_t,y_t)$ to the oracle to update its weights over the experts. Our analysis in Section~\ref{subsec:analysis-FS-SCB} and Appendix~\ref{app:proof:section4} suggest to set the exploration parameter to $\kappa = K$ and the learning rate to $\alpha=\sqrt{KT/D_T(\delta)}$, where we define $D_T(\delta)$ in Lemma~\ref{pred:error:sqalg:FS-SquareCB} and give its exact expression in Eq.~\ref{definigitonofd_t:fs} in Appendix~\ref{app:proof:lemm:sqalg:feature-selection}.

\begin{remark}[Admissible Experts]
\label{admissible-expert-feature}
It is important to note that in each round $t\in[T]$, FS-SCB only uses predictions by admissible experts, i.e.,~experts $i$ that belong to the set

\vspace{-0.15in}
\begin{small}
\begin{align}
\label{eq:addmisable-experts}
\mathcal{S}_t := \big\{i\in\mathcal{S}_{t-1} &: \langle \phi^i(x_t,a),\widehat{\theta}_t^i\rangle \leq G +   RL\sqrt{d \log\left(\frac{1 + \frac{t L^2}{\lambda_i d}}{\delta}\right)} \nonumber \\ 
&+ L \sqrt{\lambda_i} S, \;\;\;\forall a \in [K] \big\},
\end{align}
\end{small}
\vspace{-0.15in}

with $\mathcal S_0=[M]$. This is the set of experts $i$ whose predictions $f^i(x_t,a;H_t) =\langle \phi^i(x_t,a),\widehat{\theta}_t^i\rangle,\;\forall a\in[K]$ are within a bound defined by \eqref{eq:addmisable-experts}. When an expert was removed from the admissible set in a round $t$, it will remain out for the rest of the game. We discuss the technical reasons for defining this set in the proof of Lemma~\ref{lem:upper-bound:Rsq(t)} in Appendix~\ref{app:proof:lemm:sqalg:feature-selection}.
\end{remark}


\subsection{Regret Analysis of FS-SCB}
\label{subsec:analysis-FS-SCB}

We state a regret bound for FS-SCB followed by a proof sketch. The detailed proofs are all reported in Appendix~\ref{app:proof:section4}.

\begin{theorem}
\label{thm:regretbound_EXP-SquareCB}
Let Assumption~\ref{ass:boundedness-Setting2} hold and the regularization parameters $\lambda_i$, exploration parameter $\kappa$, and learning rate $\alpha$  set to the values described above. Then, for any $\delta \in [0,1/4)$, w.p.~at least $1-\delta$, the regret defined by~\eqref{eq:transfer-regret2} for FS-SCB is bounded as

\vspace{-0.15in}
\begin{small}
\begin{align}
&\mathcal R_{\text{\em FS-SCB}}(T) \leq \mathcal{O}\bigg( \sqrt{2 T \log(2/\delta)} + R L G
\nonumber \\ &
 \times \sqrt{KT (1+\log(M)) \max_{i \in [M]}\left\{\lambda_i S^2 + 4d \log\left(\frac{1 + \frac{T L^2}{\lambda_i d}}{\delta}\right)\right\}}\bigg). \nonumber
\end{align}
\end{small}
\vspace{-0.15in}

\end{theorem}

{\bf\em Proof Sketch.} The proof consists of two main steps:

{\bf Step 1.} We first need to bound the prediction error of the online regression oracle. 

\begin{lemma}
\label{pred:error:sqalg:FS-SquareCB}
For any $\delta \in (0,1/4]$, w.p.~at least $1-\delta$, we can bound the prediction error of the regression oracle as

\vspace{-0.15in}
\begin{small}
\begin{align}
& \sum_{s=1}^{t-1} \big( \widehat{y}_s(x_s,a_s) - \langle \phi^{i_*}(x_s,a_s), \theta_*^{i_*} \rangle \big)^2 \leq   
D_t(\delta) := \nonumber \\ 
& \mathcal{O}\Big( \big(1+ R^2 L^2 G^2 \log(M) \big) 
\max_{i \in [M]} \big\{ \lambda_i S^2 + 4d \log\big(\frac{1 + \frac{T L^2}{\lambda_i d}}{\delta} \big) \big\} \Big). \nonumber
\end{align}
\end{small}
\vspace{-0.15in}
\end{lemma}

The exact definition of $D_t(\delta)$ (see Eq.~\ref{definigitonofd_t:fs} in Appendix~\ref{proofoflemaofagentpredictionforsquarecb}) shows its dependence on the following two terms: {\bf 1)} an upper-bound $Q_t$ on the prediction error of the true models,

\vspace{-0.15in}
\begin{small}
\begin{equation}
\label{eq:defQ}
\max_{i_*}\sum_{s=1}^{t-1} \big(\langle \phi^{i_*}(x_s,a_s),\widehat{\theta}^{i_*}_s \rangle  - \langle \phi^{i_*}(x_s,a_s),\theta_*^{i_*} \rangle\big)^2 \leq Q_t,
\end{equation}
\end{small}
\vspace{-0.15in}

and {\bf 2)} the regret $\mathcal{R}_{\texttt{Sq}}(t)$ of the regression oracle. Thus, the proof of Lemma~\ref{pred:error:sqalg:FS-SquareCB} requires finding expressions for these quantities, which we derive them in the following lemmas. 
\begin{lemma}
\label{lem:upper-bound:Q-t}
For any $\delta \in (0,1)$, with probability at least $1-\delta$, we may write $Q_t$ defined in~\eqref{eq:defQ} as (see Eq.~\ref{final:def:q_t} in Appendix~\ref{app:lemmforleastsauqereq-t} for the exact expression)

\vspace{-0.15in}
\begin{small}
\begin{align*}
Q_t = \mathcal O\Big(  \max_{i \in [M]} \big\{ \lambda_{i} S^2 + 4 d \log \big( 1+ \frac{tL^2}{\lambda_{i} d} \big)\big\} + R^2 \log(1/\delta)\Big).
\end{align*}
\end{small}
\vspace{-0.15in}
\end{lemma}

\begin{lemma}
\label{lem:upper-bound:Rsq(t)}
For any $\delta \in (0,1)$, with probability at least $1-\delta$, we may write the regret of the regression oracle as (see Eq.~\ref{eq:finalboundrsq:fs} in Appendix~\ref{app:proof:lemm:sqalg:feature-selection} for the exact expression)

\begin{small}
\vspace{-0.15in}
\begin{align*}
& \mathcal R_{\texttt{\em Sq}}(t)  = 
\mathcal O\bigg(R^2 L^2 \log(M) \times \nonumber \\ 
& ~~~ \Big(G^2+ \max_{i\in[M]}\big\{\lambda_i S^2  +  d \log\big({1 + \frac{tL^2 }{\lambda_i d}}\big)\big\} + \log(1/\delta)\Big)\bigg).
\end{align*}
\end{small}
\vspace{-0.15in}
\end{lemma}

{\bf Step 2.} We then show how the overall regret of FS-SCB is related to the prediction error of the online regression oracle, $D_t(\delta)$, using the following lemma: 

\begin{lemma}
\label{lemm:bound-regret:expectation-counterparts}
Under the same assumptions as Theorem~\ref{thm:regretbound_EXP-SquareCB}, for any $\delta \in (0,1/4]$, with probability at least $1-\delta$, the regret of the FS-SCB algorithm is bounded as

\vspace{-0.15in}
\begin{small}
\begin{align}
&\mathcal{R}_{\text{\em FS-SCB}}(T) \leq \sqrt{2 T \log(2/\delta)} + \frac{\alpha}{4} D_T(\delta) \;+ \nonumber\\ 
&\quad\sum_{t=1}^T \sum_{a \in [K]}  p_t(a) \bigg(\langle \phi^{i_*}(x_t,a),\theta_*^{i_*}\rangle - \langle \phi^{i_*}(x_t, a^*_{t}) ,\theta_*^{i_*}\rangle \nonumber \\ 
& \qquad \qquad - \frac{\alpha}{4}\big( \widehat{y}_t(x_t,a) - \langle \phi^{i_*} (x_t,a_t), \theta^{i_*}_*\rangle \big)^2 \bigg). 
\label{regret:bound:fs-squarecb} 
\end{align}
\end{small}
\vspace{-0.15in}
\end{lemma}

Finally, we conclude the proof of Theorem~\ref{thm:regretbound_EXP-SquareCB} by bounding the last term on the RHS of~\eqref{regret:bound:fs-squarecb} using Lemma~\ref{lemma3offosterrakhlin} (see Appendix~\ref{puttinthingstogethr-boundregret-fs} for details).


\subsection{Related Work (Feature Selection)}
\label{subsec:related-work-FS}

The most straightforward solution to the feature selection problem described in Section~\ref{subsec:feature-selection} is to concatenate all models (feature maps) and build a $(M\times d)$-dimensional feature, and then search for the sparse reward parameter $\theta_*\in \mathbb R^{Md}$ with only $d$ non-zero elements. We may then solve the resulting LB problem using a sparse LB algorithm (e.g.,~\citealt{APS-2012}). This approach would result in a regret bound of $\widetilde{\mathcal{O}}(d\sqrt{MT})$, which may not be desirable when the number of models $M$ is large. 

Another approach is to use the EXP4 (or SquareCB) algorithm~\citep{Auer02NS} to obtain a regret that scales only logarithmically with $M$. If we partition the linear space of each model into $\mathcal{O}(2^d)$ predictors, we will have the total number of $\mathcal{O}(M 2^d)$ predictors. Predictor $(i,j)\in([M],[2^d])$ is associated with a linear map $\theta^{ij}\in \mathbb{R}^d$ and recommends the action $\argmax_{a\in A} \langle \phi^i(a),\theta^{ij} \rangle$. The regret of EXP4 with this set of experts is of $\widetilde{\mathcal{O}}(\sqrt{d K T \log M})$. Although this solution has logarithmic dependence on $M$, it is still not desirable, since it is not computationally efficient (requires handling $M 2^d$ predictors). 

To have computational efficiency, we can use the approach of~\citet{Maillard-Munos-2011}, but this results in a $O(T^{2/3})$ regret. They designed a model selection strategy using an EXP4 algorithm with a set of experts that are instances of the S-EXP3 algorithm of~\citet{auer2002nonstochastic}. The interesting fact is that each S-EXP3 expert is a learning algorithm and competes against a set of mappings. The overall regret of this algorithm is of $\widetilde{\mathcal O}\big(T^{2/3}(|S| K \log K)^{1/3}\sqrt{\log M}\big)$ (see \citealt[Chapter 4.2]{Bubeck-CesaBianchi-2012}). If we apply this algorithm to our setting, the resulting regret bound is of $\widetilde{\mathcal O}(T^{2/3} d^{1/3} K^{1/3}\sqrt{\log M})$. Although the algorithm is computationally more efficient than EXP4 and its regret has logarithmic dependence on $M$, it is still not desirable as its dependence on $T$ is of $\widetilde{\mathcal{O}}(T^{2/3})$, which is not optimal. 

The novelty of our results is that we propose a computationally efficient algorithm, whose regret has better dependence on $M$ and $T$, i.e.,~$\widetilde{\mathcal{O}}(\sqrt{K T \log M})$, than all the existing methods. Our FS-SCB algorithm achieves this by {\bf 1)} using a novel instantiation of SquareCB, or more precisely by constructing a proper full information algorithm (expert), and {\bf 2)} using SquareCB with a set of {\em adaptive (learning)}, and not static, least-squares experts. Note that SquareCB is a reduction that turns any online regression oracle into an algorithm for contextual bandits~\cite{FR-2020}.


More recently,~\citet{papini2021leveraging} studied a feature selection problem where the reward function is linear in \textit{all} $M$ feature maps (all models are {\em realizable}). Under this {\em stronger} assumption (than ours), they prove a regret bound that is competitive (up to a $\log M$ factor) with that of a linear bandit algorithm that uses the best feature map. More specifically, if one of the feature maps is such that a constant regret is achievable, the overall model selection strategy also achieves a constant regret. Although our focus is not on constant regret, we are able to achieve our results without requiring all models to be {\em realizable}. 


\section{Parameter Selection Algorithm}
\label{sec:param-select-alg}

We propose a UCB-style algorithm for the parameter selection setting described in Section~\ref{subsec:param-selection}, which we refer to as {\em parameter selection OFUL} (PS-OFUL). We then provide an upper-bound on its transfer regret and conclude with a discussion on the existing results related to this setting. 

Algorithm~\ref{alg:PS-OFUL} contains the pseudo-code of PS-OFUL. The novel idea in PS-OFUL is the construction of its confidence set $\mathcal C_t$ (Eq.~\ref{eq:PS-OFUL:ConfidenceSet}), which is based on the predictions $\{\widehat{y}_s\}_{s=1}^{t-1}$ by a regression oracle \texttt{SqAlg}. As described in Section~\ref{subseq:reg-oracle}, \texttt{SqAlg} is a meta algorithm that consists of $M$ learning algorithms (or experts), and its predictions $\widehat{y}_t$ are aggregates of its experts' predictions $f^i(\phi_t(a_t);H_t), \; \forall i \in [M]$. In PS-OFUL, each expert $i\in[M]$ is a {\em biased regularized least-squares} algorithm with bias $\widehat{\mu}_i$, i.e.,~our estimate of the center of the $i^{th}$ ball (model). Expert $i$ predicts the reward of the context-action $\phi_t(a_t)$ as $f^i(\phi_t(a_t);H_t) = \langle \phi_t(a_t),\widehat{\theta}^i_t \rangle$, where $\widehat{\theta}^i_t = \argmin_{\theta} \|\Phi_t^\top\theta - Y_t\|^2 + \lambda_i\|\theta - \widehat{\mu}_i\|^2$. We may write $\widehat{\theta}^i_t$ in closed-form as $\widehat{\theta}^i_t = (V_t^{\lambda_i})^{-1} \Phi_t^\top(Y_t - \Phi_t\widehat{\mu}_i) + \widehat{\mu}_i$. In these equations, the reward vector $Y_t$ and $V_t^{\lambda_i}$ are defined as in Section~\ref{sec:feature-select-alg}; $\Phi_t$ is the feature matrix, whose rows are $\phi_1(a_1),\ldots,\phi_{t-1}(a_{t-1})$; and $\lambda_i$ is the regularization parameter of expert $i$. Our analysis in Section~\ref{subsec:analysis-PS-OFUL} and Appendix~\ref{app:proofs-param-selection} suggests to set them to $\lambda_i=\frac{1}{(b_i+c_i)^2}$. 



The PS-OFUL algorithm takes the feature map $\phi$ and models $\{B(\widehat{\mu}_i,b_i+c_i)\}_{i=1}^M$ as input. At each round $t\in[T]$, it first constructs a confidence set $\mathcal C_{t-1}$ using the predictions of the regression oracle $\{\widehat{y}_s\}_{s=1}^{t-1}$.
The radius $\gamma_t(\delta)$ of the confidence set $\mathcal C_t$ is defined by two terms: {\bf 1)} the regret $\mathcal{R}_{\texttt{Sq}}(t)$ of the regression oracle \texttt{SqAlg}, defined by~\eqref{eq:regression-oracle}, and {\bf 2)} an upper-bound $U_t$ on the prediction error of the true models (i.e.,~models that contain $\theta_*$), i.e.,

\vspace{-0.15in}
\begin{equation}
\label{eq:Ut}
\max_{i\in\mathcal I_*} \; \sum_{s=1}^{t-1}\big(\langle \phi_s(a_s),\widehat{\theta}_t^{i} \rangle - \langle \phi_s(a_s),\theta_* \rangle \big)^2 \leq U_t.    
\end{equation}
\vspace{-0.15in}

The exact values of $U_t$, $\mathcal{R}_{\texttt{Sq}}(t)$, and $\gamma_t(\delta)$ come from our analysis and have been stated in Eq.~\ref{eq:final-conf-set} in Appendix~\ref{app:ptoof_themr_confidenceregion_firstsetting}. PS-OFUL then computes action $a_t$ as the one that attains the maximum optimistic reward w.r.t.~the confidence set $\mathcal C_{t-1}$. Using $a_t$, it calculates $\widehat{y}_t= \texttt{SqAlg}_t(\phi_t(a_t);H_t))$. 
As described in Section~\ref{subseq:reg-oracle}, \texttt{SqAlg} makes use of Algorithm~\ref{alg:expert-prediction} in Appendix~\ref{app:SqAlg-description} to return its prediction $\widehat{y}_t$ as an aggregate of its experts' predictions (see Remark~\ref{admissible-expert-parameter}). Finally, PS-OFUL takes action $a_t$, observes reward $y_t$, and pass the sample $(\phi_t(a_t),y_t)$ to \texttt{SqAlg}. This sample is then used within \texttt{SqAlg} to evaluate its experts and to update their weights. 

\begin{algorithm}[tb]
\caption{Parameter Selection OFUL (PS-OFUL)}
\label{alg:PS-OFUL}
\begin{algorithmic}
\STATE {\bfseries Input:} Feature Map $\phi$, Confidence Parameter $\delta$, Models $\{B(\widehat{\mu}_i,b_i+c_i)\}_{i=1}^M$
\FOR{$t=1$ {\bfseries to} $T$}
\STATE Construct the confidence set: 
  
\vspace{-0.175in}
\begin{small}
\begin{align}
\label{eq:PS-OFUL:ConfidenceSet}
\mathcal{C}_{t-1} =\left\{\theta : \sum_{s=1}^{t-1}\big(\widehat{y}_s - \langle\phi_s(a_s),\theta\rangle\big)^2 \leq \gamma_{t-1}(\delta)\right\}
\end{align}
\end{small}
\vspace{-0.125in}
  
\STATE Take action:  $a_t=\arg\max_{a\in \mathcal{A}_t}\max_{\theta\in\mathcal{C}_{t-1}} \langle\phi_t(a),\theta\rangle$
\STATE Oracle predicts: $\; \widehat{y}_t = \texttt{SqAlg}_t(\phi_t(a_t);H_t)$
\STATE Observe reward: $\;\;y_t = \langle \phi_t(a_t), \theta_* \rangle + \eta_t$ 
\STATE Update \texttt{SqAlg} with $(\phi_t(a_t),y_t)$;
\ENDFOR
\end{algorithmic}
\end{algorithm}

\begin{remark}[Admissible Experts]
\label{admissible-expert-parameter}
Similar to FS-SCB, in each round $t\in[T]$, PS-OFUL only uses predictions by admissible experts, i.e.,~experts $i$ that belong to the set

\vspace{-0.15in}
\begin{small}
\begin{align}
\label{ps-oful-admissible-expert}
\mathcal{S}_t := \big\{i\in\mathcal{S}_{t-1} &: \langle \phi_t(a_t),\widehat{\theta}_t^i\rangle \leq G + RL\sqrt{d \log\left(\frac{1 + \frac{t L^2}{\lambda_i d}}{\delta} \right)} \nonumber \\ 
&+ L\sqrt{\lambda_i} (b_i + c_i) \big\}, 
\end{align}
\end{small}
\vspace{-0.15in}

with $\mathcal S_0=[M]$. This is the set of experts $i$ whose prediction $f^i(\phi_t(a_t);H_t) =\langle \phi_t(a_t),\widehat{\theta}_t^i\rangle$ is within a bound defined by \eqref{ps-oful-admissible-expert}. When an expert was removed from the admissible set in a round $t$, it will remain out for the rest of the game. We discuss the technical reasons for defining this set in the proof of Lemma~\ref{lem:PS-OFUL-SqAlg-reg} in Appendix~\ref{app:thm:upperbound-regret-oracle}.
\end{remark}


\subsection{Regret Analysis of PS-OFUL}
\label{subsec:analysis-PS-OFUL}

We state a regret bound for PS-OFUL followed by a proof sketch. The detailed proofs are all reported in Appendix~\ref{app:proofs-param-selection}. 


\begin{theorem}
\label{thm:PS-OFUL-regret}
Let Assumption~\ref{ass:boundedness-Setting1} hold and $\lambda_i=\frac{1}{(b_i+c_i)^2}\geq 1,\;\forall i\in[M]$. Then, for any $\delta \in (0,1/4]$, with probability at least $1-\delta$, the transfer-regret defined by~\eqref{eq:transfer-regret} of PS-OFUL is bounded as

\vspace{-0.15in}
\begin{small}
\begin{align}
\label{eq:regret-param-select}
\mathcal{R}(T) &= \mathcal{O}\bigg(dRL\max\{1,G\}\sqrt{1 + \log(M)} \\
&\hspace{-0.125in} \times\sqrt{T\log\big(1 + \frac{T}{d}\big) \log\Big(\frac{1 + \frac{T L^2 \max_{i\in[M]}(b_{i} + c_{i})^2}{ d}}{\delta}\Big)} \bigg). \nonumber
\end{align}
\end{small}
\vspace{-0.15in}
%
\end{theorem}
 
{\bf\em Proof Sketch.} The proof consists of two main steps. 

{\bf Step~1.} We first fully specify the confidence set $\mathcal C_t$ and prove its validity i.e.,~$\mathbb P(\theta_* \in \mathcal C_t) \geq 1 - \delta,\;\forall t\in [T]$. 

\begin{theorem}
\label{thm:PS-OFUL-confidence-set}
Under the same assumptions as Theorem~\ref{thm:PS-OFUL-regret}, the radius $\gamma_t(\delta)$ of the confidence set $\mathcal C_t$ is fully specified by Eq.~\ref{eq:final-conf-set} in Appendix~\ref{app:ptoof_themr_confidenceregion_firstsetting}. Moreover, for any $\delta \in (0,1/4]$, with probability at least $1-\delta$, the true reward parameter $\theta_*$ lies in $\mathcal{C}_t$, i.e., $\mathbb{P}\left(\theta_* \in \mathcal{C}_t\right) \geq 1-\delta$. 
\end{theorem}

The definition of $\gamma_t(\delta)$ in Eq.~\ref{eq:final-conf-set} shows its dependence on $U_t$ and $\mathcal R_{\texttt{Sq}}(t)$, defined by~\eqref{eq:Ut} and~\eqref{eq:regression-oracle}, respectively. Thus, the proof of Thm.~\ref{thm:PS-OFUL-confidence-set} requires finding expressions for these quantities, which we derive them in the following lemmas.  

\begin{lemma}
\label{lem:PS-OFUL-U_t}
Setting $\lambda_i=\frac{1}{(b_i+c_i)^2},\;\forall i\in[M]$, with probability $1-\delta$, we may write $U_t$, defined by~\eqref{eq:Ut}, as (see Eq.~\ref{eq:Ut-Def} in Appendix~\ref{app:subsec:proofoflemaofboudingthexperts} for the exact expression)

\vspace{-0.15in}
\begin{small}
\begin{align*}
U_t = \mathcal O\Big(dR^2 \log\big(\frac{1 + \frac{tL^2\max_{i\in[M]}(b_i+c_i)^2}{d}}{\delta}\big)\Big).
\end{align*}
\end{small}
\vspace{-0.15in}
\end{lemma}

\begin{lemma}
\label{lem:PS-OFUL-SqAlg-reg}
Setting $\lambda_i=\frac{1}{(b_i+c_i)^2},\;\forall i\in[M]$, with probability $1-\delta$, we may write $\mathcal R_{\texttt{\em Sq}}(t)$, defined by~\eqref{eq:regression-oracle}, as (see Eq.~\ref{eq:temp07} in Appendix~\ref{app:thm:upperbound-regret-oracle} for the exact expression)

\vspace{-0.15in}
\begin{small}
\begin{align*}
\mathcal R_{\texttt{\em Sq}}(t) = 
\mathcal O\Big(d R^2 L^2 \log(M)\log\big(\frac{1 + \frac{tL^2\max_{i\in[M]}(b_i+c_i)^2}{d}}{\delta}\big)\Big).
\end{align*}
\end{small}
\vspace{-0.15in}
\end{lemma}


{\bf Step~2.} We then show how the regret is related to the confidence sets using the following lemma:

\begin{lemma}
\label{lem:PS-OFUL-reg:sumup}
Under the same assumptions as Theorem~\ref{thm:PS-OFUL-regret}, for any $\delta \in (0,1/4]$, with probability at least $1-\delta$, the regret of PS-OFUL is bounded as 

\vspace{-0.15in}
\begin{small}
\begin{align}
\label{eq:lem5}
& \mathcal{R}_{\text{\em PS-OFUL}}(T) \leq  2 Gd + \\ 
&\qquad\quad  2\max\{1,G\}\sqrt{2dT\log\big(1 + \frac{T}{d} \big)\max_{d<t\leq T} {\gamma_t(\delta)}}. \nonumber
\end{align}
\end{small}
\vspace{-0.15in}
\end{lemma}
We conclude the proof of Theorem~\ref{thm:PS-OFUL-regret} by plugging the confidence radius $\gamma_t(\delta)$ computed in Theorem~\ref{thm:PS-OFUL-confidence-set} (Eq.~\ref{eq:final-conf-set} in Appendix~\ref{app:ptoof_themr_confidenceregion_firstsetting}) into the regret bound~\eqref{eq:lem5}. 





\subsection{Related Work (Parameter Selection)}
\label{subsec:related-work-PS}

\citet{cella2020meta} and~\citet{moradipari2022multi} studied {\em meta learning} in stochastic linear bandit (LB), where the agent solves a sequence of LB problems, whose reward parameters $\theta_*$ are drawn from an unknown distribution $\rho$ of bounded support in $\mathbb R^d$. For each LB task, the agent is given an estimate of the mean of the distribution $\rho$ and an upper-bound of its error, and its goal is to minimize the transfer regret $\mathcal R(T,\rho)=\mathbb E_{\theta_*\sim\rho}\big[\mathbb E[\mathcal R(T,\theta_*)]\big]$. Their proposed algorithms assume knowing the variance term $\text{Var}_h=\mathbb E_{\theta_*\sim\rho}\big[\|\theta_* - h\|^2\big]$, for any $h\in\mathbb R^d$, in order to properly set their regularization parameter $\lambda$. Thus, the parameter selection setting studied in our paper can be seen as an extension of their transfer learning setting to multiple ($M$) models. Moreover, we allow the reward parameter of the new LB problem $\theta_*$ to be selected arbitrarily from the $M$ models, and consider a worst-case transfer regret (see Eq.~\ref{eq:transfer-regret}) for our algorithm (instead of a regret in expectation w.r.t.~$\rho$). Despite these differences, our setting is similar to theirs as we are also given an estimate of the center of each model $\widehat\mu_i$, together with an upper-bound on its error $c_i$, plus the radius $b_i$ of each model. Also similar to their results, our analysis clearly shows the importance of the choice of the regularization parameters, $\lambda_i=1/(b_i+c_i)^2$, for obtaining a regret bound that only logarithmically depends on the maximum model uncertainty, i.e.,~$\max_{i\in[M]}(b_i+c_i)^2$. 

Our parameter selection setting is also related to {\em latent} bandits~\citep{Maillard14LB,HKZCAB-2020} in which identifying the true latent variable is analogous to finding the correct model. The latest work in this area is by~\citet{HKZCAB-2020} in which the agent faces a $K$-armed LB problem selected from a set of $M$ known $K$-dimensional reward vectors. They proposed UCB and TS algorithms for this setting and showed that their regret (Bayes regret in case of TS) are bounded as $3 M + 2 T \varepsilon + 2R\sqrt{6 M T\log T}$, where the reward vectors are known up to an error of $\varepsilon$. Comparing to their results, the regret of PS-OFUL in~\eqref{eq:regret-param-select} has a better dependence on the number of models, $\sqrt{\log M}$ vs.~$M$, and the model uncertainty, $\sqrt{\log(\max_{i\in[M]}(b_i+c_i)^2)}$ vs.~$\varepsilon$. However, the number of actions $K$ does not appear in their bound, while the bound of PS-OFUL will have a $\sqrt{K}$ factor when applied to $K$-armed bandit problems. If the objective is to have a better scaling in $K$, we can use a different bandit model selection strategy, called {\em regret balancing}~\citep{APP-2020,PDGB-2020}, to obtain an improved regret that scales as $\min\{\varepsilon T + \sqrt{MT}, \sqrt{K M T}\}$ (see Appendix~\ref{app:latent bandits} for details).

In another closely related work, \citet{Hong-TS-2022} approach a similar problem by initializing TS with a prior that is a mixture of $M$ distributions. They prove a Bayes regret bound for their algorithm in case of Gaussian mixtures that has $\sqrt{M}$ dependence on the number of models and $\sqrt{\max_{i\in[M]}\sigma^2_{0,i}}$ dependence on the maximum variance of the Gaussian priors. Both these dependences are logarithmic $\sqrt{\log M}$ and $\sqrt{\log(\max_{i\in[M]}(b_i+c_i)^2)}$ in the regret of PS-OFUL.




\begin{figure*}[t]
 \centering
     \includegraphics[width=0.72\linewidth]{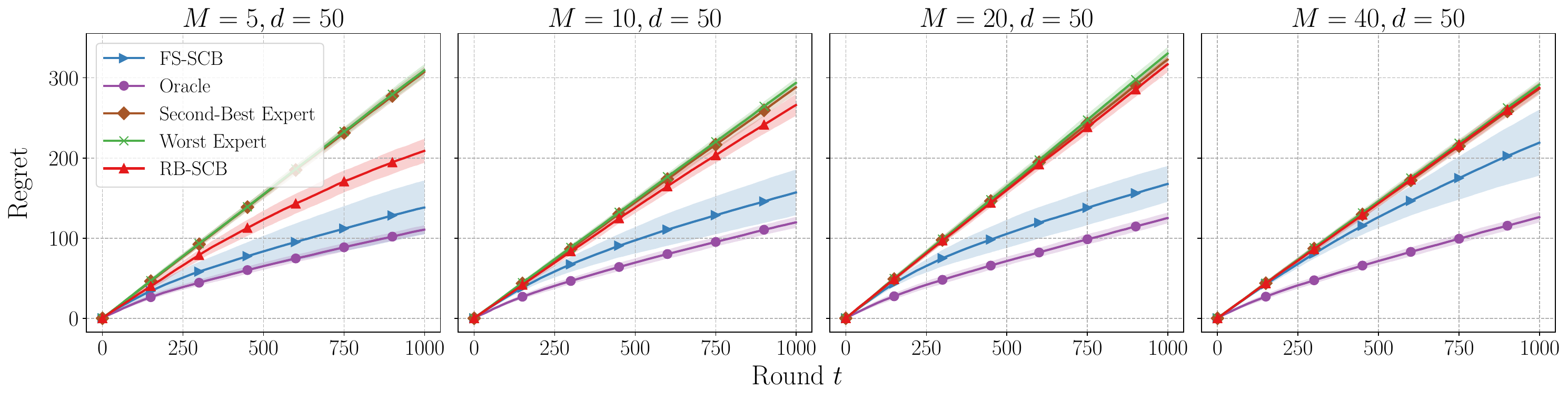}
      \includegraphics[width=0.72\linewidth]{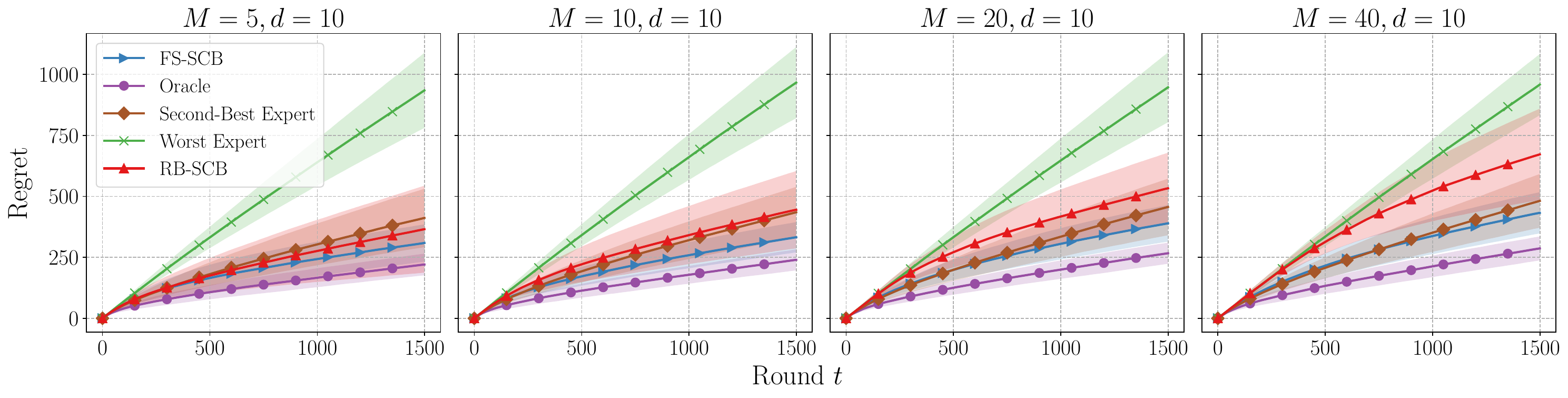}
     \caption{Feature selection in the synthetic LB problem {\em (top)} and MNIST {\em (bottom)}. The regrets are averaged over $100$ LB problems.}
     \label{fig:synthetic_fs-scb}
\end{figure*}

\begin{figure*}[t]
     \centering
     \includegraphics[width=0.385\linewidth]{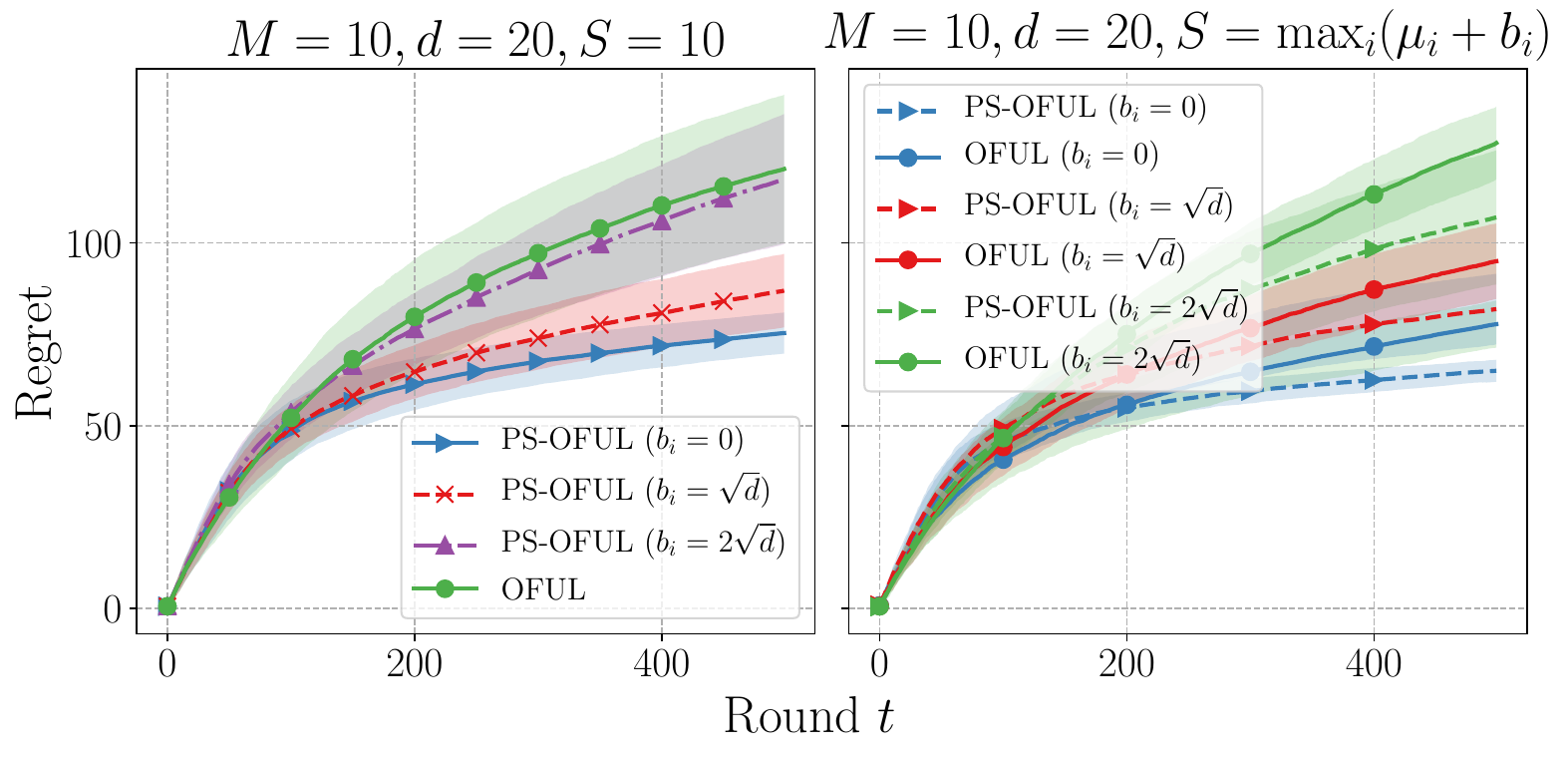} 
         \includegraphics[width=0.55\linewidth]{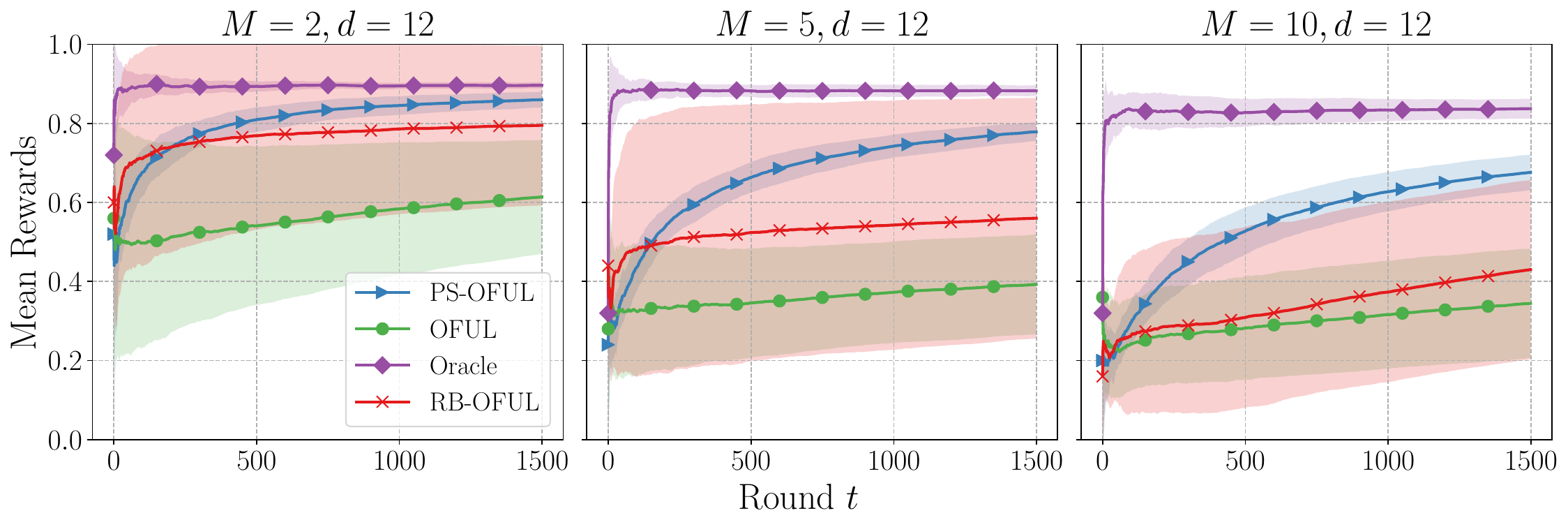}
         \caption{Parameter selection in the synthetic LB problem {\em (left)} and CIFAR-10 {\em (right)}. Results are averaged over $50$ runs.}
        \label{onefig:ps_experiments}
\end{figure*}


\section{Experiments}
\label{sec:experiments}

We evaluate the performances of our FS-SCB and PS-OFUL algorithms using a synthetic LB problem and image classification problems: MNIST~\citep{lecun1998gradient} and CIFAR-10~\citep{krizhevsky2009learning}. We report the details of our experimental setup and additional results in Appendix~\ref{app:experiment_details}.

\textbf{Feature Selection \begin{small}(Synthetic)\end{small}:} We first sample the parameter of the linear bandit problem from a $d=50$ dimensional Gaussian with variance 0.01: $\theta_*\sim{\cal N}(0,0.01 I_d)$. We generate all feature maps, $\{\phi^i(a)\}_{i=1}^M$, by sampling $10,000$ vectors from the Gaussian with mean $\theta_*$ and covariance $0.1I_d$, i.e.,~$\phi^i(a)\sim{\cal N}(\theta_*,0.1 I_d)$, for $a=1,\ldots,10,000$. This implies that all $M$ feature maps have the same bias. We set $\phi^1(\cdot)$ to be the {\em true} feature map. At each round $t\in[T]$, the learner is given an action set consist of 10 numbers from ${\cal A}=\{1,2,\dots,10,000\}$. The reward of each action $a$ is $\langle \phi^1(a),\theta_\star\rangle+\eta_t$, where $\eta_t\sim {\cal U}[-0.5,0.5]$.

\textbf{Feature Selection (MNIST):} We train a convolutional neural network (CNN) with $M$ different number of epochs on MNIST data, and use their second layer to the last as our $d=10$-dimensional feature maps $\{\phi^i\}_{i=1}^M$. These feature maps have test accuracy between $20\%$ (worst model) and $97\%$ (best model). We set the best one as true model $\phi^{i_*}$. For each class $s\in\mathcal S=\{0,\ldots,9\}$, we fit a linear model, given the feature map $\phi^{i_*}$, and obtain parameters $\{\theta^{i_*}_s\}_{s=0}^9$. At the beginning of each LB task, we select a class $s_*\in\mathcal S$ uniformly at random and set its parameter to $\theta^{i_*}_{s_*}$. At each round $t\in[T]$, the learner is given an action set consists of $10$ images, one from class $s_*$ and the rest randomly selected from the other classes. The reward of each action $a$ is defined as $\langle\phi^{i_*}(a),\theta^{i_*}_{s_*}\rangle+\eta_t\in[0,1]$, where $\phi^{i_*}(a)$ is the application of the feature map $\phi^{i_*}$ to the image corresponding to action $a$ and $\eta_t\sim\mathcal U[-0.5,0.5]$ is the noise. 

In Figure~\ref{fig:synthetic_fs-scb}, 
we compare the regret of our FS-SCB algorithm for different number of models $M$ with a regret balancing algorithm that uses SquareCB baselines (RB-SCB), and three SquareCB algorithms that use the best (Oracle), second-best (with test accuracy $84\%$ for MNIST), and worst feature maps. 
The results in Figure~\ref{fig:synthetic_fs-scb} 
show that {\bf 1)} FS-SCB always performs between the best and second-best experts, {\bf 2)} the regret of FS-SCB that scales as $\sqrt{\log M}$ is close to RB-SCB (scales as $\sqrt{M}$) for small $M$, but gets much better as $M$ grows, and {\bf 3)} RB-SCB has much higher variance than the other algorithms in MNIST. 

\begin{figure*}[t]
    \centering
    \includegraphics[width=.72\linewidth]{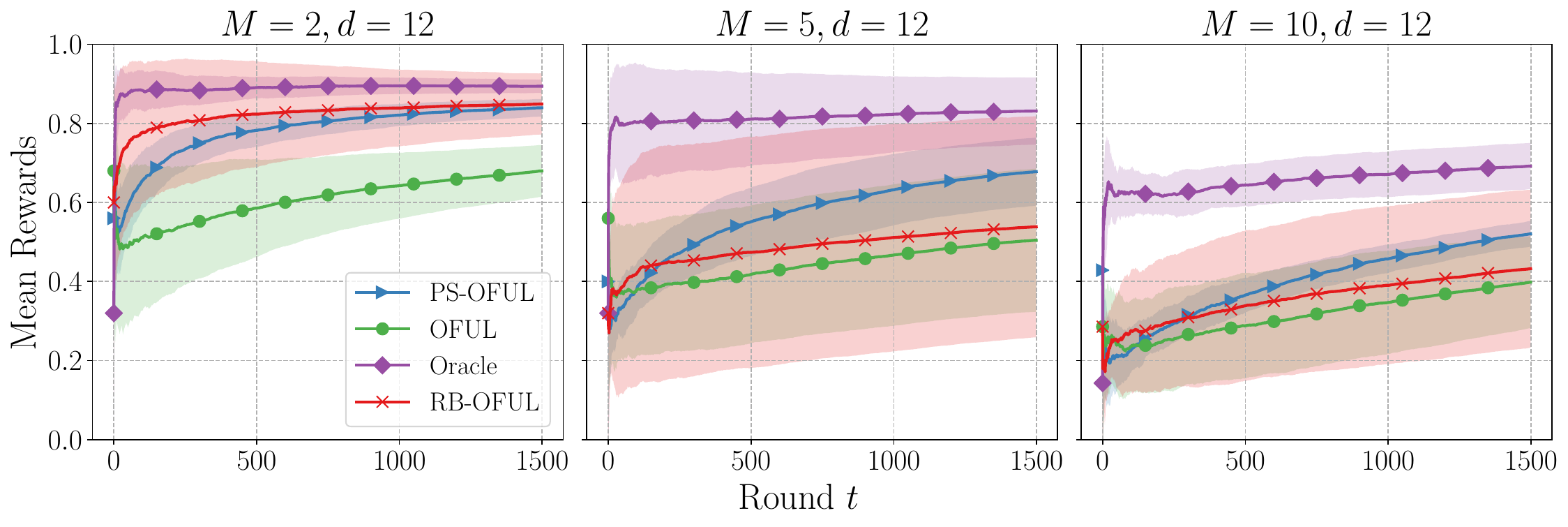}
    \caption{Parameter selection in CIFAR-10 with models less accurate than those in Figure~\ref{onefig:ps_experiments} {\em (right)}. The results are averaged over 50 runs.}
    \label{fig:app_cifar10_ps_largervar}
\end{figure*}

\textbf{Parameter Selection \begin{small}(Synthetic)\end{small}:} We first sample the center of $M=10$ balls from a $d=20$-dimensional Gaussian, i.e.,~$\{\mu_i\}_{i=1}^M\sim\mathcal{N}(0,I_d)$, and set their radii to $b_i=b,\;\forall i\in[M]$. At the beginning of each LB task, we select a model $i_*\in[M]$ uniformly at random, and then sample the problem's parameter from its ball, i.e.,~$\theta_*\sim B(\mu_{i_*},b_{i_*})$. The action set in each round $t\in[T]$ consists of $10$ vectors $\{\phi_t(a_j)\}_{j=1}^{10}\sim\mathcal{N}(0,0.01 I_d)$, and the reward of the selected action $a_t$ is defined as $\langle\phi_t(a_t),\theta_*\rangle+\eta_t,\;\eta_t\sim{\cal U}[-0.5,0.5]$. Figure~\ref{onefig:ps_experiments} {\em (left)} compares the regret of our PS-OFUL algorithm with OFUL~\citep{abbasi2011improved} for different sizes of the balls $b\in\{0,\sqrt{d},2\sqrt{d}\}$. We run OFUL with the upper-bounds $\|\theta_*\|_2\leq S=10$ and $S=\max_i(\mu_i+b_i)$ on the reward parameter. Note that the second bound is tighter and shows the best performance of OFUL. Our results indicate that the regret of PS-OFUL is better than OFUL, and gets closer to it as we increase the size of the balls from $b=0$ to $b=2\sqrt{d}\approx 9$. This clearly shows the potential advantage of transfer (PS-OFUL) over individual (OFUL) learning. 

\textbf{Parameter Selection \begin{small}(CIFAR-10)\end{small}:}
We modify the EfficientNetV2-S network~\citep{tan2021efficientnetv2} by adding a layer of $d=12$ neurons before the last layer and fine-tuning it on CIFAR-10 dataset. We then select this $d$-dimensional layer as our feature map $\phi$. To define our $M$ models (balls), we sample $100M$ datasets of size $500$. For each dataset, we randomly select a class $s_*\in[M]$, assign reward $1$ to images from $s_*$ and $0$ to other images, and fit a linear model to it to obtain a parameter vector. Finally, we fit a Gaussian mixture model with $M$ components to these $100M$ parameter vectors and use the means and covariances of the resulting clusters as the center and radii of our $M$ models (balls). At the beginning of each LB task, we select a class $s_*\in[M]$ uniformly at random. In each round $t\in[T]$, the learner is given an action set consists of $10$ images, one from class $s_*$ and the rest randomly selected from the other classes. The learner receives a reward from Ber$(0.9)$, if it selects the image from class $s_*$, and from Ber$(0.1)$, otherwise.  

In Figure~\ref{onefig:ps_experiments} {\em (right)}, we compare the mean reward of PS-OFUL for different values of $M$ with a regret balancing algorithm that uses OFUL baselines (RB-OFUL)~\citep{APP-2020}, OFUL (individual learning), and BIAS-OFUL~\citep{cella2020meta} with bias being the center of the true model (Oracle). 
The results show {\bf 1)} the good performance of PS-OFUL, {\bf 2)} the performance of PS-OFUL gets better than RB-OFUL as $M$ grows ($\sqrt{\log M}$ vs.~$\sqrt{M}$ scaling), {\bf 3)} the large variance of RB-OFUL, especially in comparison to PS-OFUL, and finally {\bf 4)} the advantage of transfer (PS-OFUL) over individual (OFUL) learning. 

In order to show the impact of the model accuracy (the accuracy of the center of the balls and their radii) on the performance of the algorithms, we defined a less accurate set of $M$ models (balls) using $10M$ datasets of size 50 (as opposed to $100M$ datasets of size 500 used in the results reported in Figure~\ref{onefig:ps_experiments} {\em (right)}). In Figure~\ref{fig:app_cifar10_ps_largervar}, we compare the mean reward of PS-OFUL for different number of models $M$ with RB-OFUL, OFUL, and BIAS-OFUL. The results indicate that with decreasing in the accuracy of the models, the performance of PS-OFUL and RB-OFUL get closer to that for OFUL. 




\section{Conclusions }

We studied two model selection settings in LB, where the mean reward is linear in at least one of $M$ models {\em (feature selection)}, and where the reward parameter is arbitrarily selected from $M$ misspecified models {\em (parameter selection)}. We derived computationally efficient algorithms in these settings that are based on reductions from bandits to full-information problems, and proved regret bounds with desirable dependence on the horizon and number of models. 
An interesting future direction is to extend our results to the meta learning and learning-to-learn setting, where the agent starts with $M$ models, and instead of solving a single LB problem, has to solve $N$ of them one after another.


\section{Acknowledgement}

This work is partially supported by NSF grant 1847096. 

\newpage
\bibliography{bibfile.bib}

\begin{thebibliography}{40}
\providecommand{\natexlab}[1]{#1}
\providecommand{\url}[1]{\texttt{#1}}
\expandafter\ifx\csname urlstyle\endcsname\relax
  \providecommand{\doi}[1]{doi: #1}\else
  \providecommand{\doi}{doi: \begingroup \urlstyle{rm}\Url}\fi

\bibitem[Abbasi-Yadkori et~al.(2011)Abbasi-Yadkori, P{\'a}l, and
  Szepesv{\'a}ri]{abbasi2011improved}
Abbasi-Yadkori, Y., P{\'a}l, D., and Szepesv{\'a}ri, C.
\newblock Improved algorithms for linear stochastic bandits.
\newblock In \emph{Advances in Neural Information Processing Systems}, pp.\
  2312--2320, 2011.

\bibitem[Abbasi-Yadkori et~al.(2012)Abbasi-Yadkori, Pal, and
  Szepesvari]{APS-2012}
Abbasi-Yadkori, Y., Pal, D., and Szepesvari, C.
\newblock Online-to-confidence-set conversions and application to sparse
  stochastic bandits.
\newblock In \emph{Proceedings of the Fifteenth International Conference on
  Artificial Intelligence and Statistics}, 2012.

\bibitem[Abbasi-Yadkori et~al.(2020)Abbasi-Yadkori, Pacchiano, and
  Phan]{APP-2020}
Abbasi-Yadkori, Y., Pacchiano, A., and Phan, M.
\newblock Regret balancing for bandit and {RL} model selection.
\newblock \emph{arXiv:2006.05491}, 2020.

\bibitem[Abe \& Long(1999)Abe and Long]{Abe99AR}
Abe, N. and Long, P.
\newblock Associative reinforcement learning using linear probabilistic
  concepts.
\newblock In \emph{ICML}, 1999.

\bibitem[Abeille et~al.(2017)Abeille, Lazaric, et~al.]{abeille2017linear}
Abeille, M., Lazaric, A., et~al.
\newblock Linear thompson sampling revisited.
\newblock \emph{Electronic Journal of Statistics}, 11\penalty0 (2):\penalty0
  5165--5197, 2017.

\bibitem[Agarwal et~al.(2017)Agarwal, Luo, Neyshabur, and Schapire]{ALNS-2017}
Agarwal, A., Luo, H., Neyshabur, B., and Schapire, R.~E.
\newblock Corralling a band of bandit algorithms.
\newblock In \emph{COLT}, 2017.

\bibitem[Agrawal \& Goyal(2013)Agrawal and Goyal]{agrawal2013thompson}
Agrawal, S. and Goyal, N.
\newblock Thompson sampling for contextual bandits with linear payoffs.
\newblock In \emph{International Conference on Machine Learning}, pp.\
  127--135, 2013.

\bibitem[Auer et~al.(2002{\natexlab{a}})Auer, Cesa-Bianchi, and Fischer]{Auer}
Auer, P., Cesa-Bianchi, N., and Fischer, P.
\newblock Finite-time analysis of the multiarmed bandit problem.
\newblock \emph{Machine Learning}, 47\penalty0 (2-3):\penalty0 235--256,
  2002{\natexlab{a}}.

\bibitem[Auer et~al.(2002{\natexlab{b}})Auer, Cesa-Bianchi, Freund, and
  Schapire]{Auer02NS}
Auer, P., Cesa-Bianchi, N., Freund, Y., and Schapire, R.
\newblock The non-stochastic multi-armed bandit problem.
\newblock \emph{SIAM Journal of Computing}, 2002{\natexlab{b}}.

\bibitem[Auer et~al.(2002{\natexlab{c}})Auer, Cesa-Bianchi, Freund, and
  Schapire]{auer2002nonstochastic}
Auer, P., Cesa-Bianchi, N., Freund, Y., and Schapire, R.~E.
\newblock The nonstochastic multiarmed bandit problem.
\newblock \emph{SIAM journal on computing}, 32\penalty0 (1):\penalty0 48--77,
  2002{\natexlab{c}}.

\bibitem[Bubeck \& Cesa-Bianchi(2012)Bubeck and
  Cesa-Bianchi]{Bubeck-CesaBianchi-2012}
Bubeck, S. and Cesa-Bianchi, N.
\newblock Regret analysis of stochastic and nonstochastic multiarmed bandit
  problems.
\newblock \emph{Foundations and Trends in Machine Learning}, 2012.

\bibitem[Cella et~al.(2020)Cella, Lazaric, and Pontil]{cella2020meta}
Cella, L., Lazaric, A., and Pontil, M.
\newblock Meta-learning with stochastic linear bandits.
\newblock In \emph{International Conference on Machine Learning}, pp.\
  1360--1370, 2020.

\bibitem[Cesa-Bianchi \& Lugosi(2006)Cesa-Bianchi and Lugosi]{CBL-2006}
Cesa-Bianchi, N. and Lugosi, G.
\newblock \emph{Prediction, learning, and games}.
\newblock Cambridge University Press, 2006.

\bibitem[Cutkosky et~al.(2021)Cutkosky, Dann, Das, Gentile, Pacchiano, and
  Purohit]{cutkosky2021dynamic}
Cutkosky, A., Dann, C., Das, A., Gentile, C., Pacchiano, A., and Purohit, M.
\newblock Dynamic balancing for model selection in bandits and rl.
\newblock In \emph{International Conference on Machine Learning}, pp.\
  2276--2285. PMLR, 2021.

\bibitem[Dani et~al.(2008)Dani, Hayes, and Kakade]{Dani08stochasticlinear}
Dani, V., Hayes, T., and Kakade, S.~M.
\newblock Stochastic linear optimization under bandit feedback.
\newblock \emph{21st Annual Conference on Learning Theory}, 2008.

\bibitem[Foster \& Rakhlin(2020)Foster and Rakhlin]{FR-2020}
Foster, D. and Rakhlin, A.
\newblock Beyond ucb: Optimal and efficient contextual bandits with regression
  oracles.
\newblock In \emph{International Conference on Machine Learning}, pp.\
  3199--3210, 2020.

\bibitem[Foster et~al.(2020)Foster, Gentile, Mohri, and Zimmert]{Foster21AM}
Foster, D., Gentile, C., Mohri, M., and Zimmert, J.
\newblock Adapting to misspecification in contextual bandits.
\newblock In \emph{Advances in Neural Information Processing Systems 35}, pp.\
  11478--11489, 2020.

\bibitem[Haussler et~al.(1998)Haussler, Kivinen, and
  Warmuth]{Herbster-Kivinen-Warmuth-1998}
Haussler, D., Kivinen, J., and Warmuth, M.~K.
\newblock Sequential prediction of individual sequences under general loss
  functions.
\newblock \emph{IEEE Trans. Inform. Theory}, pp.\  1906--1925, 1998.

\bibitem[Hong et~al.(2020)Hong, Kveton, Zaheer, Chow, Ahmed, and
  Boutilier]{HKZCAB-2020}
Hong, J., Kveton, B., Zaheer, M., Chow, Y., Ahmed, A., and Boutilier, C.
\newblock Latent bandits revisited.
\newblock In \emph{NeurIPS}, 2020.

\bibitem[Hong et~al.(2022)Hong, Kveton, Zaheer, Ghavamzadeh, and
  Boutilier]{Hong-TS-2022}
Hong, J., Kveton, B., Zaheer, M., Ghavamzadeh, M., and Boutilier, C.
\newblock Thompson sampling with a mixture prior.
\newblock In \emph{AISTATS}, 2022.

\bibitem[Krizhevsky et~al.(2009)]{krizhevsky2009learning}
Krizhevsky, A. et~al.
\newblock Learning multiple layers of features from tiny images.
\newblock 2009.

\bibitem[Lai \& Robbins(1985)Lai and Robbins]{lai85asymptotically}
Lai, T. and Robbins, H.
\newblock Asymptotically efficient adaptive allocation rules.
\newblock \emph{Advances in Applied Mathematics}, 6\penalty0 (1):\penalty0
  4--22, 1985.

\bibitem[Lattimore \& Szepesvari(2020)Lattimore and Szepesvari]{LS-2020}
Lattimore, T. and Szepesvari, C.
\newblock \emph{Bandit Algorithms}.
\newblock Cambridge University Press, 2020.

\bibitem[LeCun et~al.(1998)LeCun, Bottou, Bengio, and
  Haffner]{lecun1998gradient}
LeCun, Y., Bottou, L., Bengio, Y., and Haffner, P.
\newblock Gradient-based learning applied to document recognition.
\newblock \emph{Proceedings of the IEEE}, 86\penalty0 (11):\penalty0
  2278--2324, 1998.

\bibitem[Maillard \& Mannor(2014)Maillard and Mannor]{Maillard14LB}
Maillard, O. and Mannor, S.
\newblock Latent bandits.
\newblock In \emph{ICML}, 2014.

\bibitem[Maillard \& Munos(2011)Maillard and Munos]{Maillard-Munos-2011}
Maillard, O. and Munos, R.
\newblock Adaptive bandits: Towards the best history-dependent strategy.
\newblock In \emph{AISTATS}, 2011.

\bibitem[Moradipari et~al.(2018)Moradipari, Silva, and
  Alizadeh]{moradipari2018learning}
Moradipari, A., Silva, C., and Alizadeh, M.
\newblock Learning to dynamically price electricity demand based on multi-armed
  bandits.
\newblock In \emph{2018 IEEE Global Conference on Signal and Information
  Processing (GlobalSIP)}, pp.\  917--921. IEEE, 2018.

\bibitem[Moradipari et~al.(2020{\natexlab{a}})Moradipari, Alizadeh, and
  Thrampoulidis]{moradipari2020linear}
Moradipari, A., Alizadeh, M., and Thrampoulidis, C.
\newblock Linear thompson sampling under unknown linear constraints.
\newblock In \emph{ICASSP 2020-2020 IEEE International Conference on Acoustics,
  Speech and Signal Processing (ICASSP)}, pp.\  3392--3396. IEEE,
  2020{\natexlab{a}}.

\bibitem[Moradipari et~al.(2020{\natexlab{b}})Moradipari, Thrampoulidis, and
  Alizadeh]{moradipari2020stage}
Moradipari, A., Thrampoulidis, C., and Alizadeh, M.
\newblock Stage-wise conservative linear bandits.
\newblock \emph{Advances in neural information processing systems},
  33:\penalty0 11191--11201, 2020{\natexlab{b}}.

\bibitem[Moradipari et~al.(2021)Moradipari, Amani, Alizadeh, and
  Thrampoulidis]{moradipari2021safe}
Moradipari, A., Amani, S., Alizadeh, M., and Thrampoulidis, C.
\newblock Safe linear thompson sampling with side information.
\newblock \emph{IEEE Transactions on Signal Processing}, 69:\penalty0
  3755--3767, 2021.

\bibitem[Moradipari et~al.(2022{\natexlab{a}})Moradipari, Ghavamzadeh, and
  Alizadeh]{moradipari2022collaborative}
Moradipari, A., Ghavamzadeh, M., and Alizadeh, M.
\newblock Collaborative multi-agent stochastic linear bandits.
\newblock \emph{arXiv preprint arXiv:2205.06331}, 2022{\natexlab{a}}.

\bibitem[Moradipari et~al.(2022{\natexlab{b}})Moradipari, Ghavamzadeh,
  Rajabzadeh, Thrampoulidis, and Alizadeh]{moradipari2022multi}
Moradipari, A., Ghavamzadeh, M., Rajabzadeh, T., Thrampoulidis, C., and
  Alizadeh, M.
\newblock Multi-environment meta-learning in stochastic linear bandits.
\newblock \emph{arXiv preprint arXiv:2205.06326}, 2022{\natexlab{b}}.

\bibitem[Pacchiano et~al.(2020{\natexlab{a}})Pacchiano, Dann, Gentile, and
  Bartlett]{PDGB-2020}
Pacchiano, A., Dann, C., Gentile, C., and Bartlett, P.
\newblock Regret bound balancing and elimination for model selection in bandits
  and rl.
\newblock \emph{arXiv:2012.13045}, 2020{\natexlab{a}}.

\bibitem[Pacchiano et~al.(2020{\natexlab{b}})Pacchiano, Phan, Abbasi-Yadkori,
  Rao, Zimmert, Lattimore, and Szepesvari]{pacchiano2020model}
Pacchiano, A., Phan, M., Abbasi-Yadkori, Y., Rao, A., Zimmert, J., Lattimore,
  T., and Szepesvari, C.
\newblock Model selection in contextual stochastic bandit problems.
\newblock \emph{arXiv preprint arXiv:2003.01704}, 2020{\natexlab{b}}.

\bibitem[Papini et~al.(2021)Papini, Tirinzoni, Restelli, Lazaric, and
  Pirotta]{papini2021leveraging}
Papini, M., Tirinzoni, A., Restelli, M., Lazaric, A., and Pirotta, M.
\newblock Leveraging good representations in linear contextual bandits.
\newblock In \emph{ICML}, 2021.

\bibitem[Rusmevichientong \& Tsitsiklis(2010)Rusmevichientong and
  Tsitsiklis]{rusmevichientong2010linearly}
Rusmevichientong, P. and Tsitsiklis, J.~N.
\newblock Linearly parameterized bandits.
\newblock \emph{Mathematics of Operations Research}, 35\penalty0 (2):\penalty0
  395--411, 2010.

\bibitem[Russo \& Van~Roy(2014)Russo and Van~Roy]{russo}
Russo, D. and Van~Roy, B.
\newblock Learning to optimize via posterior sampling.
\newblock \emph{Mathematics of Operations Research}, 39\penalty0 (4):\penalty0
  1221--1243, 2014.
\newblock \doi{10.1287/moor.2014.0650}.

\bibitem[Tan \& Le(2019)Tan and Le]{tan2019efficientnet}
Tan, M. and Le, Q.
\newblock Efficientnet: Rethinking model scaling for convolutional neural
  networks.
\newblock In \emph{International Conference on Machine Learning}, pp.\
  6105--6114. PMLR, 2019.

\bibitem[Tan \& Le(2021)Tan and Le]{tan2021efficientnetv2}
Tan, M. and Le, Q.~V.
\newblock Efficientnetv2: Smaller models and faster training.
\newblock \emph{arXiv preprint arXiv:2104.00298}, 2021.

\bibitem[Thompson(1933)]{thompson1933likelihood}
Thompson, W.
\newblock On the likelihood that one unknown probability exceeds another in
  view of the evidence of two samples.
\newblock \emph{Biometrika}, 25\penalty0 (3/4):\penalty0 285--294, 1933.

\end{thebibliography}
\bibliographystyle{icml2022}


\newpage
\appendix
\onecolumn


\section{Sequential Prediction Algorithm} 
\label{app:SqAlg-description}

The sequential prediction algorithm \texttt{SqAlg} uses the following algorithm from~\cite{Herbster-Kivinen-Warmuth-1998} (also see~\citealt[Chapter~3]{CBL-2006}) to aggregate its experts' predictions. Algorithm~\ref{alg:expert-prediction} takes the observations $y_t$ and experts' predictions $f^i_t(H_t)$ that are bounded in the known range $[\beta,\beta+\ell ]$ as input. It first scales these input to the range $[0,1]$ and uses its current weights for the experts to generate its own prediction $\widehat{y}_t$.

The performance of \texttt{SqAlg} is evaluated as the accuracy (in terms of square loss) of its prediction w.r.t.~the accuracy of the prediction by the best expert in the set, i.e.,
\begin{equation}
\label{appx:eq:regression-oracle1}
\sum_{t=1}^T(\widehat{y}_t - y_t)^2 - \min_{i\in [M]}\sum_{t=1}^T(f^i_t(H_t) - y_t)^2 \leq {\mathcal R}_{\texttt{Sq}}(T).
\end{equation}
We call this the regret of \texttt{SqAlg} and denote it by ${\mathcal R}_{\texttt{Sq}}(T)$.~\citet{Herbster-Kivinen-Warmuth-1998} prove the following bound for ${\mathcal R}_{\texttt{Sq}}(T)$, which we use in the analysis of our algorithms. 

\begin{algorithm}[h]
  \caption{Sequential Prediction with Expert Advice}
  \label{alg:expert-prediction}
\begin{algorithmic}
  \STATE \textbf{Input:} $\ell$ and $\beta\qquad$ {\em (experts' predictions $f^i_t(H_t)$ are bounded in the known range $[\beta,\beta+\ell]$)}
  \STATE {\bfseries Initialization:} Set the weight $w_{1,i} = 1$ for all experts $i\in[M]$ 
  \FOR{$t=1$ {\bfseries to} $T$}
  \STATE Receive predictions $f^i_t(H_t)$ by experts $i\in \mathcal{S}_{t-1}$
  \STATE Remove experts whose predictions are out of bound and construct the new set of admissible experts $\mathcal{S}_t$ {\em (see Remarks~\ref{admissible-expert-feature} and~\ref{admissible-expert-parameter})}
  \STATE Scale experts' predictions $\;h_{i,t} = \frac{f^i_t(H_t) - \beta}{\ell},\;\forall i\in \mathcal{S}_t$
  
  \STATE Set $\;v_{t,i} = \frac{w_{t,i}}{W},\;\forall i\in\mathcal{S}_t$, where $W = \sum_{i\in \mathcal{S}_t} w_{t,i}$ 
  
  \STATE \textbf{Prediction:} Compute: 
\begin{align*}
\Delta(0) = \frac{-1}{2} \log \left( \sum_{i\in \mathcal{S}_t} v_{t,i} e^{-2 h_{i,t}^2} \right) \qquad,\qquad \Delta(1) = \frac{-1}{2} \log \left( \sum_{i\in \mathcal{S}_t} v_{t,i} e^{-2(1 - h_{i,t})^2} \right) 
\end{align*}

 \STATE Predict a value $\widehat{y}'_t$ that satisfies the following conditions:
\begin{align*}
(\widehat{y}'_t)^2 \leq \Delta(0) \qquad , \qquad (1-\widehat{y}'_t)^2 \leq \Delta(1).
\end{align*}
\STATE \textbf{Update:} Observing reward $y_t$, scale it as $y'_t=\frac{y_t - a}{\ell}$, and update the experts' weights
\begin{align}
w_{t+1,i} = w_{t,i} e^{-2 (y'_t - h_{i,t})^2} \label{update:weight_exprtt}
\end{align}
\STATE Return prediction $\;\widehat{y}_t  = \beta + \ell \widehat{y}'_t$
  \ENDFOR
\end{algorithmic}
\end{algorithm}

\begin{proposition}[Theorem~4.2 in~\citealt{Herbster-Kivinen-Warmuth-1998}]
\label{app:prop:husslerreulst}
For any arbitrary sequence $\big\{\big(\{f^i_t(H_t)\}_{i=1}^M,\widehat{y}_t,y_t\big)\big\}_{t=1}^T$ in which the experts' predictions $\big\{\{f^i_t(H_t)\}_{i=1}^M\big\}_{t=1}^T$ and observations $\{y_t\}_{t=1}^T$ are all bounded in $[\beta, \beta+\ell]$, the regret defined by~\eqref{appx:eq:regression-oracle1} of Algorithm~\ref{alg:expert-prediction} is  bounded as 
\begin{align*}
\mathcal{R}_{\texttt{\em Sq}}(T) \leq 2 \ell^2 \log M.
\end{align*}
\end{proposition}

Here we use the fact that $|\mathcal{S}_t| \leq M, \; \forall t \in [T]$. 


\newpage 

\section{Proofs of Section \ref{sec:feature-select-alg}}\label{app:proof:section4}

In this section, we first provide a brief overview for the steps of our proof. Then, we provide the proofs of lemmas used in Section~\ref{sec:feature-select-alg}.

The performance analysis of the FS-SCB algorithm requires two steps. First, we control the sum of the prediction error of the agent. 
Second, we show how the regret is related to the prediction error of the agent, and then we bound the regret.

{\bf Step 1.} To control the sum of the prediction error of the agent $D_t$, we need to find two  upper bounds: 1) an upper bound on the prediction error of the true model $i_*$ whose identity is unknown to the agent $Q_t$; 2) an upper bound on the regret caused by the online regression oracle $\mathcal{R}_{\texttt{Sq}}$

First, in Lemma~\ref{lem:upper-bound:Q-t}, we bound the sum of the prediction error of the true model as
\begin{align}
    \sum_{s=1}^{t-1} \left(\langle \phi^{i_*}(x_s,a_s), \widehat{\theta}_s^{i_*} \rangle - \langle \phi^{i_*}(x_s,a_s), \theta_*^{i_*} \rangle \right)^2 \leq Q_t \label{app:lemm:defofqtforboudn}
\end{align} 
where  
\begin{align*}
Q_t = 1 + 2 \left( \max_{i \in [M]} \Big\{\lambda_{i} S^2 + 4 d \log \big( 1+ \frac{tL^2}{\lambda_{i} d} \big) \Big\}\right) + 32 R^2 \log \left( \frac{R\sqrt{8} + \sqrt{1+  \max_{i \in [M]} \big\{ \lambda_{i} S^2 + 4 d \log \big( 1+ \frac{tL^2}{\lambda_{i} d} \big)\big\} }}{\delta} \right). 
\end{align*}

Next, in Lemma~\ref{lem:upper-bound:Rsq(t)}, we provide a high probability upper-bound on the regret caused by the online regression oracle as
\begin{align*}
    \mathcal{R}_{\texttt{Sq}}(t) \leq 8(\log M) R^2 L^2 \bigg( G^2+ \max_{i\in[M]}\{\lambda_i S^2  +  d \log\Big({1 + \frac{tL^2 }{\lambda_i d}} \Big)\} + \log(1/\delta) \bigg).
\end{align*}
Then, in Lemma~\ref{pred:error:sqalg:FS-SquareCB}, we show the following upper bound on the sum of prediction error of the agent:
\begin{align}
D_t(\delta) \leq 1 &+ 2 {\mathcal R}_{\texttt{Sq}}(t) + 2 Q_t + 4 R \sqrt{2(1+Q_t) \log\big(\frac{\sqrt{1+Q_t}}{\delta}\big)} \nonumber \\ 
& + 32 R^2 \log \left( \frac{R\sqrt{8} + \sqrt{1+ {\mathcal R}_{\texttt{Sq}}(t) + Q_t + 2 R \sqrt{2(1+Q_t) \log\big(\frac{\sqrt{1+Q_t}}{\delta}\big)}}}{\delta}\right). 
\end{align}
{\bf Step 2.} First in Lemma~\ref{lemm:bound-regret:expectation-counterparts}, we show how the regret is related to the prediction error of the agent using the Azuma's inequality, i.e.,
\begin{align}
\mathcal{R}_{\text{FS-SCB}}(T) &\leq \sqrt{2T \log(2/\delta)} + \alpha D_T(\delta) \\ 
&+ \sum_{t=1}^T \sum_{a\in [K]} p_{t}(x_t,a) \left(\langle \phi^{i_*}(x_t,a^*_{t}) ,\theta_*^{i_*}\rangle - \langle \phi^{i_*}(x_t,a),\theta_*^{i_*}\rangle  - \frac{\alpha}{4} \left(\widehat{y}_t(x_t,a) - \langle \phi^{i_*}(x_t,a),\theta_*^{i_*}\rangle \right)^2 \right). \nonumber  \label{app:eq:lemma:boundtheregertforlemma2ofrakh}
\end{align}
Then in Appendix~\ref{puttinthingstogethr-boundregret-fs}, we put everything together  and complete the proof. 


\subsection{Proof of Lemma~\ref{lem:upper-bound:Q-t}}
\label{app:lemmforleastsauqereq-t}

At each round $t$, each expert $i_* \in \mathcal{I}_*$ estimates its reward parameter as 
\begin{align}
\widehat{\theta}_t^{i_*} = \arg\min_\theta  \norm{({\Phi_t^{i_*}})^\top \theta - Y_t}_2^2 + \lambda_{i_*} \norm{\theta}_2^2.
\end{align} 
%
Let $V_t^{\lambda_{i_*}} = \lambda_{i_*} I + \sum_{s=1}^{t-1} \phi^{i_*}(x_s,a_s)  {\phi^{i_*}(x_s,a_s)}^\top$. From the standard least-squares analysis, we have 
\begin{align*}
\sum_{s=1}^{t-1} \big(\langle \phi^{i_*}(x_s,a_s), \widehat{\theta}^{i_*}_s  \rangle- y_s\big)^2 - \sum_{s=1}^{t-1} \big(\langle \phi^{i_*}(x_s,a_s), \theta_*^{i_*} \rangle- y_s\big)^2 \leq \lambda_{i_*} \norm{\theta_*^{i_*}}_2^2 + 2  \sum_{s=1}^{t-1} \big\langle  {\phi^{i_*}(x_s,a_s)}^\top ,(V_s^{{\lambda_{i_*}}})^{-1} \phi^{i_*}(x_s,a_s) \big\rangle.
\end{align*}
Therefore,  we can write: 
\begin{align}
\sum_{s=1}^{t-1} \big(\langle \phi^{i_*}(x_s,a_s), \widehat{\theta}^{i_*}_s  - \theta_*^{i_*} \rangle\big)^2  \leq \lambda_{i_*} \norm{\theta_*^{i_*}}_2^2 + 2  \log \left( \frac{\det(V_t^{\lambda_{i_*}})}{\det(\lambda_{i_*} I)} \right) + 2 \sum_{s=1}^{t-1} \eta_s\big(\langle \phi^{i_*}(x_s,a_s), \widehat{\theta}^{i_*}_s - \theta_*^{i_*} \rangle\big). 
\label{forboundingtheq_tfirstone}
\end{align} 
The last term on the RHS of~\eqref{forboundingtheq_tfirstone} can be bounded using Proposition~\ref{aux:prop:sum-marting-bound} in Appendix~\ref{app:auxily-tools} as 
\begin{align}
& \left| \sum_{s=1}^{t-1} \eta_s\big(\langle \phi^{i_*}(x_s,a_s), \widehat{\theta}^{i_*}_s - \theta_*^{i_*} \rangle\big)\right| \leq \nonumber\\ 
& \qquad\qquad R \sqrt{2 \left(1 + \sum_{s=1}^{t-1} \big(\langle \phi^{i_*}(x_s,a_s), \widehat{\theta}^{i_*}_s - \theta_*^{i_*} \rangle \big)^2\right) \log \left(\frac{1 + \sum_{s=1}^{t-1} \big(\langle \phi^{i_*}(x_s,a_s), \widehat{\theta}^{i_*}_s - \theta_*^{i_*} \rangle \big)^2}{\delta} \right)}. \label{first_lemma_fsscb_machinary}
\end{align} 
Define $u = \sqrt{1 + \sum_{k=1}^{t-1} \big(\langle \phi^{i_*}(x_k,a_k), \widehat{\theta}^{i_*}_k  - \theta_*^{i_*} \rangle\big)^2}$, $v=1+\lambda_{i_*} \norm{\theta_*^{\lambda_{i_*}}}_2^2 + 2  \log \left( \frac{\det(V_t^{i_*})}{\det(\lambda_{i_*} I)} \right)  $, and $w = 2R\sqrt{2\log(s/\delta)}$.  It is easy to see that \eqref{first_lemma_fsscb_machinary} can be written in the form of $u^2 \leq v + uq$. Then, by applying Lemma~\ref{aux:lemma:square-rrot-trick} in Appendix~\ref{app:auxily-tools}, we may write $u \leq \sqrt{v} + w$. Substituting for $w$, we can get $u \leq \sqrt{v} + 2R\sqrt{2\log(u/\delta)}$. Then, by Lemma~\ref{aux:lemma:log-trick} in Appendix~\ref{app:auxily-tools}, for $\delta \in (0,1/4]$, we have
\begin{align*}
u \leq \sqrt{v} + 4 R \sqrt{\log \left(\frac{2\sqrt{2}R + \sqrt{v}}{\delta} \right)},
\end{align*}
which using the inequality $(a + b)^2 \leq 2a^2 + 2b^2$, for any $a$ and $b$, we can write it as 
\begin{align*}
u^2 \leq 2v + 32 R^2 \log \left(\frac{2\sqrt{2}R + \sqrt{v}}{\delta} \right). 
\end{align*} 
Finally, we substitute $u$ and $v$, and subtract $1$ from both sides, and for $\delta \in (0,1/4]$, we obtain 
\begin{align}
\sum_{s=1}^{t-1} \big(\langle \phi^{i_*}(x_s,a_s), \widehat{\theta}^{i_*}_s  - \theta_*^{i_*} \rangle\big)^2 & \leq 1 + 2\lambda_{i_*} \norm{\theta_*^{i_*}}_2^2 + 4 \log \left( \frac{\det(V_t^{\lambda_{i_*}})}{\det(\lambda_{i_*} I)} \right)  \nonumber \\ 
& + 32 R^2 \log \left( \frac{R\sqrt{8} + \sqrt{1+\lambda \norm{\theta_*^{i_*}}_2^2 + 2  \log \left( \frac{\det(V_t^{\lambda_{i_*}})}{\det(\lambda_{i_*} I)} \right) }}{\delta} \right). 
\label{app:lemm:def:Q_t}
\end{align}
We know that $\norm{\theta_*^{i_*}}_2^2 \leq S^2$. Moreover, by Lemma \ref{aux:lemm:det-teace-ineq} in Appendix~\ref{app:auxily-tools}, we can bound the term $\log \left( \frac{\det(V_t^{\lambda_{i_*}})}{\det(\lambda_{i_*} I)} \right)$. Replacing these in~\eqref{app:lemm:def:Q_t}, we may write 
\begin{align}
\sum_{s=1}^{t-1} \big(\langle \phi^{i_*}(x_s,a_s), \widehat{\theta}^{i_*}_s  - \theta_*^{i_*} \rangle\big)^2 & \leq 1 +   2\lambda_{i_*} S^2 + 8 d \log \left( 1+ \frac{tL^2}{\lambda_{i_*} d} \right)  \nonumber \\ 
&+ 32 R^2 \log \left( \frac{R\sqrt{8} + \sqrt{1+\lambda_{i_*} S^2 + 4 d \log \left(1 + \frac{tL^2}{\lambda_{i_*} d} \right) }}{\delta} \right). \label{drving_q_t_fortrue_model}
\end{align}
Since the algorithm does not know the identity of $i_*$, we derive an expression for $Q_t$, and conclude the proof by replacing $i_*$ with the maximum over all $i \in [M]$ in~\eqref{drving_q_t_fortrue_model} as
\begin{align}
\sum_{s=1}^{t-1} \big(\langle \phi^{i_*}(x_s,a_s),\widehat{\theta}^{i_*}_s - \theta_*^{i_*} \rangle\big)^2 &\leq 1 + 2 \left(\max_{i \in [M]}\Big\{ \lambda_{i} S^2 + 4 d \log \big( 1+ \frac{tL^2}{\lambda_{i} d} \big)\Big\}\right) \nonumber \\ 
&+ 32 R^2 \log \left( \frac{R\sqrt{8} + \sqrt{1+  \max_{i \in [M]} \big\{ \lambda_{i} S^2 + 4 d \log \big( 1+ \frac{tL^2}{\lambda_{i} d} \big) \big\} }}{\delta} \right) := Q_t.   \label{final:def:q_t}
\end{align}


\subsection{Proof of Lemma~\ref{lem:upper-bound:Rsq(t)}} 
\label{app:proof:lemm:sqalg:feature-selection}

To bound the regret $\mathcal{R}_{\texttt{Sq}}(t)$ of the regression oracle \texttt{SqAlg}, similar to the proof of Lemma~\ref{lem:PS-OFUL-SqAlg-reg} in Appendix~\ref{app:thm:upperbound-regret-oracle}, we show the reward signals and the experts' predictions are bounded with high probability. Then, we use Proposition~\ref{app:prop:husslerreulst} in Appendix~\ref{app:SqAlg-description} to complete the proof. 

From~\eqref{eq:reward-bounded}, according to Assumption~\ref{ass:boundedness-Setting2}, we have $\langle \phi^{i_*}(x,a), \theta_*^{i_*}\rangle \leq LS$, for all $x \in \mathcal{X}$, $a\in[k]$, and $i\in[M]$. Hence, with probability at least $1-\delta$, we have
\begin{align}
\label{bound:reward:fssqaurecb}
y_t \in \left[-\left(G + R\sqrt{2\log(2/\delta)} \right), \left(G + R\sqrt{2\log(2/\delta)}\right)\right].
\end{align}
Next we bound the predictions of the experts that FS-SCB considers in its prediction. To do so, we first show an upper bound on the prediction of the any true model $i_*$. In particular, we can write for $t \in [T]$ and $\forall a \in [K]$:
\begin{align}
    \left| \langle \phi^{i_*}(x_t,a_t), \widehat{\theta}_t^{i_*} \rangle \right| &= \left| \langle \phi^{i_*}(x_t,a_t), \theta_*^{i_*} \rangle +\langle \phi^{i_*}(x_t,a_t), \widehat{\theta}_t^{i_*} - \theta_*^{i_*} \rangle \right| \nonumber \\&
\stackrel{\text{(a)}}{\leq} \left|\langle \phi^{i_*}(x_t,a_t), \theta_*^{i_*} \rangle \right| + \left| \langle \phi^{i_*}(x_t,a_t), \widehat{\theta}_t^{i_*} - \theta_*^{i_*} \rangle \right| \nonumber \\& \stackrel{\text{(b)}}{\leq}
G + \norm{\phi^{i_*}(x_t,a_t)}_{(V_t^{\lambda_{i_*}})^{-1}} \left(\norm{\Phi_t^{i_*} \eta_t}_{(V_t^{\lambda_{i_*}})^{-1}} + \sqrt{\lambda_{i_*}} S \right), \label{uppernpund_prediction_fsscb}
\end{align}
{\bf (a)} It results from a triangular inequality.  {\bf (b)} This is because of the Assumption~\ref{ass:boundedness-Setting2}, and the fact that the true model is linearly realizable, we can apply Theorem~2 in \citet{abbasi2011improved}. 
Then, we use Theorem~1 in~\citet{abbasi2011improved} and standard matrix analysis together with our assumption that $\norm{\phi^{i_*}(x_t,a_t)} \leq L$, and bound the terms on the RHS of~\eqref{uppernpund_prediction_fsscb} with high probability as
\begin{align}
\label{eq:fsscb:expertnoisebound}
\norm{\Phi_t^{i_*} \eta_t}_{(V_t^{\lambda_{i_*}})^{-1}} \leq R \sqrt{2 \log\left( \frac{\sqrt{\det(V_t^{\lambda_{i_*}})}}{\delta\sqrt{\det(\lambda_{i_*} I)}} \right)},
\end{align}
and
\begin{align}
\label{eq:fasb:featurebound}
\norm{\phi^{i_*}(x_t,a_t)}_{(V_t^{\lambda_{i_*}})^{-1}} \leq \frac{\norm{\phi^{i_*}(x_t,a_t)}}{\sqrt{\lambda_{\text{min}}(V_t^{\lambda_{i_*}})}} \leq \frac{L}{\sqrt{\lambda_{i_*}}} \leq L, 
\end{align}
where $\lambda_{\text{min}}(V_t^{\lambda_{i_*}})$ is the smallest eigenvalue of the matrix $V_t^{\lambda_{i_*}}$. In the last step of~\eqref{eq:fasb:featurebound}, we use the fact that $\lambda_{i} \geq 1,\;\forall i\in[M]$. Putting Eqs.~\ref{uppernpund_prediction_fsscb},~\ref{eq:fsscb:expertnoisebound}, and~\ref{eq:fasb:featurebound} together, with probability at least $1-\delta$, we have
\begin{align}
\label{eq:fsscb:expert_optimal_bounds}
\left| \langle \phi^{i_*}(x_t,a_t), \widehat{\theta}_t^{i_*} \rangle \right| \leq G + {R L} \sqrt{2 \log\left(\frac{\sqrt{\det(V_t^{\lambda_{i_*}})}}{\delta\sqrt{\det(\lambda_{i_*} I)}} \right)} + L\sqrt{\lambda_i}S.
\end{align}
Using Lemma~\ref{aux:lemm:det-teace-ineq} in Appendix~\ref{app:auxily-tools}, we may write~\eqref{eq:fsscb:expert_optimal_bounds} as
\begin{align}
\label{eq:fs-scb:final:expert:bound}
\left| \langle \phi^{i_*}(x_t,a_t), \widehat{\theta}_t^{i_*} \rangle \right| \leq G + RL\sqrt{d \log\left(\frac{1 + \frac{t L^2}{\lambda_{i_*} d}}{\delta} \right)} + L\sqrt{\lambda_{i_*}} S.
\end{align}
FS-SCB employees this idea that at any time step $t \in [T]$, any potentially true model (i.e., linearly realizable) should have a similar bound on its prediction.
To do so, the set of admissible experts, $\mathcal{S}_t$, only considers experts that have the following bound on their prediction at each time $t \in [T]$ and $\forall a \in [K]$ as:
\begin{align}
\left| \langle \phi^i(x_t,a_t), \widehat{\theta}_t^i \rangle \right| \leq G + RL\sqrt{d \log\left(\frac{1 + \frac{t L^2}{\lambda_i d}}{\delta} \right)} + L \sqrt{\lambda_i} S.
\label{upper:bound:exp-pred-fssqcb}
\end{align}
If at some time step $t$, this bound does not hold for any expert $i$, then the algorithm simply eliminates that expert from the set of admissible experts, since that model is not a true model (i.e., the reward is not in the linear span of the prediction of that expert), and that expert will remain out for the rest of the game.
Then, we may bound the range of the prediction of each expert $i \in \mathcal{S}_t$ at round $t\in[T]$ as
\begin{align}
\label{range:pred:exp:fssquareceb}
\langle \phi^i(x_t,a_t), \widehat{\theta}_t^i \rangle \in
\bigg[-\bigg(G & + RL \sqrt{d \log\Big(\frac{1 + \frac{tL^2 }{\lambda_i d}}{\delta}\Big)} + L\sqrt{\lambda_i} S \bigg),\bigg(G  + RL \sqrt{d \log\Big(\frac{1 + \frac{tL^2 }{\lambda_i d}}{\delta}\Big)} + L \sqrt{\lambda_i} S \bigg)\bigg].
\end{align}
Putting together~\eqref{bound:reward:fssqaurecb} and~\eqref{range:pred:exp:fssquareceb}, we conclude that for all rounds $t \in [T]$ and experts $i \in \mathcal{S}_t$, with probability at least $1-\delta$, the reward $y_t$ and the expert's predictions $f_t^i(H_t)$ are in the range $[\beta, \beta+\ell]$ for 
\begin{equation}
\label{eq:boudningrangefor-fs}
\beta = -\bigg(G + RL \sqrt{d \log\Big(\frac{1 + \frac{tL^2 }{\lambda_i d}}{\delta}\Big)} + L \sqrt{\lambda_i} S \bigg), \qquad \ell = 2\bigg(G + RL \sqrt{d \log\Big(\frac{1 + \frac{tL^2 }{\lambda_i d}}{\delta}\Big)} + L \sqrt{\lambda_i} S \bigg).
\end{equation} 
Using Proposition~\ref{app:prop:husslerreulst} in Appendix~\ref{app:SqAlg-description} with the bound on the observations and predictions in~\eqref{eq:boudningrangefor-fs}, with probability at least $1-\delta$, we obtain the following regret bound for \texttt{SqAlg}: 
\begin{align}
\label{eq:finalboundrsq:fs}
\mathcal{R}_{\texttt{Sq}}(t) = 8R^2 L^2\log (M) \left( G^2 +  \max_{i\in[M]}\left\{\lambda_i S^2  +  d \log\Big({1 + \frac{tL^2 }{\lambda_i d}}\Big) \right\} + \log(1/\delta)\right),
\end{align} 
in which we use the fact that for $a,b>0,\;(a + b)^2 \leq 2 a^2 + 2b^2$. This concludes our proof.


\subsection{Proof of Lemma~\ref{pred:error:sqalg:FS-SquareCB}}
\label{proofoflemaofagentpredictionforsquarecb}

Here, we bound the sum of the square loss of the oracle predictions, i.e., 
\begin{align}
\sum_{s=1}^{t-1} & \left( \widehat{y}_s(x_s,a_s) -  \langle \phi^{i_*}(x_s,a_s), \theta_*^{i_*} \rangle \right)^2 \leq D_t(\delta).   
\end{align}
We know that $y_t = \langle \phi^{i_*}(x_t,a_t), \theta^{i_*}_* \rangle  + \eta_t$. Hence we can write 
\begin{align}
& \big(\widehat{y}_t(x_t,a_t) - y_t\big)^2 -  \big(\langle \phi^{i_*}(x_t,a_t), \widehat{\theta}_t^{i_*} \rangle -y_t\big)^2 = \nonumber \\ 
&~~~~~~~~~~ \big( \widehat{y}_t(x_t,a_t) - \langle \phi^{i_*}(x_t,a_t), {\theta}^{i_*}_* \rangle  - \eta_t \big)^2 - \big( \langle \phi^{i_*}(x_t,a_t), \widehat{\theta}_t^{i_*} \rangle - \langle \phi^{i_*}(x_t,a_t), \theta^{i_*}_* \rangle - \eta_t \big)^2 \nonumber \\ &= \left( \widehat{y}_t(x_t,a_t) - \langle \phi^{i_*}(x_t,a_t), {\theta}^{i_*}_* \rangle \right)^2 - \big( \langle \phi^{i_*}(x_t,a_t), \widehat{\theta}_t^{i_*} \rangle - \langle \phi^{i_*}(x_t,a_t), \theta^{i_*}_* \rangle\big)^2 \nonumber \\ 
&~~~~~~~~~~ + 2 \eta_t\big( \langle \phi^{i_*}(x_t,a_t), \widehat{\theta}_t^{i_*} \rangle - \widehat{y}_t(x_t,a_t)\big) \nonumber \\ 
&= \left( \widehat{y}_t(x_t,a_t) - \langle \phi^{i_*}(x_t,a_t), {\theta}^{i_*}_* \rangle \right)^2 - \big( \langle \phi^{i_*}(x_t,a_t), \widehat{\theta}_t^{i_*} \rangle - \langle \phi^{i_*}(x_t,a_t), \theta^{i_*}_* \rangle\big)^2 \nonumber \\ 
&~~~~~~~~~~ + 2 \eta_t\big(\langle \phi^{i_*}(x_t,a_t), \widehat{\theta}_t^{i_*} \rangle - \langle \phi^{i_*}(x_t,a_t), \theta^{i_*}_* \rangle\big) + 2 \eta_t \big(\langle \phi^{i_*}(x_t,a_t), \theta^{i_*}_* \rangle - \widehat{y}_t(x_t,a_t)\big). \label{openingthe:fs:forbounds:d-t}
\end{align}
Then, from Proposition~\ref{aux:prop:sum-marting-bound} in Appendix~\ref{app:auxily-tools}, with probability at least $1-\delta$, we have 
\begin{align}
& \left|\sum_{s=1}^{t-1} \eta_s\big( \langle \phi^{i_*}(x_s,a_s), \widehat{\theta}_s^{i_*} - \theta^{i_*}_* \rangle\big) \right| \leq \nonumber \\ 
& \qquad\qquad R \sqrt{2 \left( 1 + \sum_{s=1}^{t-1}\big( \langle \phi^{i_*}(x_s,a_s), \widehat{\theta}_s^{i_*} - \theta_*^{i_*} \rangle \big)^2 \right) \log \left( \frac{\sqrt{1 + \sum_{s=1}^{t-1} \big( \langle \phi^{i_*}(x_s,a_s), \widehat{\theta}_s^{i_*} - \theta_*^{i_*} \rangle \big)^2  }}{\delta}\right) }, 
\label{boundinthenoisytermforsquarefor:FS}
\end{align}
and 
\begin{align}
& \left|\sum_{s=1}^{t-1} \eta_s \big(\langle \phi^{i_*}(x_s,a_s), \theta^{i_*}_* \rangle - \widehat{y}_s(x_s,a_s)\big) \right| \leq \nonumber \\ 
& \qquad\qquad R \sqrt{2 \left(1 + \sum_{s=1}^{t-1} \big(\langle \phi^{i_*}(x_s,a_s), \theta_*^{i_*} \rangle - \widehat{y}_s \big)^2\right) \log \left( \frac{\sqrt{1 + \sum_{s=1}^{t-1}\big( \langle \phi^{i_*}(x_s,a_s), \theta_*^{i_*} \rangle - \widehat{y}_s \big)^2 }}{\delta} \right)}. 
\label{boundingsecomdnpisytermforsquare:FS}
\end{align}
Using~\eqref{boundinthenoisytermforsquarefor:FS} and~\eqref{boundingsecomdnpisytermforsquare:FS}, the upper-bound $\mathcal{R}_{\texttt{Sq}}(t)$ from~\eqref{eq:finalboundrsq:fs} in Appendix~\ref{app:proof:lemm:sqalg:feature-selection}, and the upper-bound $Q_t$ on the square error of the prediction of the true model in~\eqref{final:def:q_t} in~Appendix~\ref{app:lemmforleastsauqereq-t}, we may write~\eqref{openingthe:fs:forbounds:d-t} as 
\begin{align}
& \sum_{s=1}^{t-1} \left(\widehat{y}_s(x_s,a_s) - \langle \phi^{i_*}(x_s,a_s), \theta_*^{i_*} \rangle \right)^2 \leq {\mathcal R}_{\texttt{Sq}}(t) + Q_t + 2 R \sqrt{2(1+Q_t) \log\left(\frac{\sqrt{1+Q_t}}{\delta}\right)} \nonumber\\ 
& \qquad + 2R \sqrt{2 \Big(1 + \sum_{s=1}^{t-1} \left( \langle \phi^{i_*}(x_s,a_s), \theta_*^{i_*} \rangle - \widehat{y}_s(x_s,a_s) \right)^2 \Big) \log \left( \frac{\sqrt{1 + \sum_{s=1}^{t-1} \left( \langle \phi^{i_*}(a_s), \theta_*^{i_*} \rangle - \widehat{y}_s(a_s) \right)^2 }}{\delta} \right)}. \label{FS-scb-proof-machinary}
\end{align}
Let $u = \sqrt{1 + \sum_{k=1}^{t-1}  \left( \widehat{y}_k(x_k,a_k) -  \langle \phi^{i_*}(x_k,a_k), \theta_*^{i_*} \rangle \right)^2}$, $v= 1+ {\mathcal R}_{\texttt{Sq}}(t) + Q_t + 2 R \sqrt{2(1+Q_t) \log(\frac{\sqrt{1+Q_t}}{\delta})}$, and $q = 2 R\sqrt{2 \log(s/\delta)}$. 
Then, following the same machinery as the one in the proof of Lemma~\ref{lem:upper-bound:Q-t} in Section~\ref{app:lemmforleastsauqereq-t}, and with the use of Lemmas~\ref{aux:lemma:square-rrot-trick} and~\ref{aux:lemma:log-trick}, for $\delta \in (0,1/4]$, with probability at least $1-\delta$, we have
\begin{align}
\sum_{s=1}^{t-1} & \left( \widehat{y}_s(x_s,a_s) - \langle \phi^{i_*}(x_s,a_s), \theta_*^{i_*} \rangle \right)^2  \leq 1 + 2 {\mathcal R}_{\texttt{Sq}}(t) + 2 Q_t + 4 R \sqrt{2(1+Q_t) \log\left(\frac{\sqrt{1+Q_t}}{\delta}\right)} \nonumber \\ 
& + 32 R^2 \log \left( \frac{R\sqrt{8} + \sqrt{1+ {\mathcal R}_{\texttt{Sq}}(t) + Q_t + 2 R \sqrt{2(1+Q_t) \log\left(\frac{\sqrt{1+Q_t}}{\delta}\right)}}}{\delta}\right) := D_t(\delta), 
\label{definigitonofd_t:fs}
\end{align}   
where
\begin{align*}
Q_t = 1 +   2 & \left(\max_{i \in [M]}\left\{ \lambda_{i} S^2 + 4 d \log \left( 1+ \frac{tL^2}{\lambda_{i} d} \right) \right\}\right) + 32 R^2 \log \left( \frac{R\sqrt{8} + \sqrt{1+  \max_{i \in [M]} \left\{ \lambda_{i} S^2 + 4 d \log \big( 1+ \frac{tL^2}{\lambda_{i} d} \big) \right\}}}{\delta} \right),
\end{align*} 
and
\begin{align*}
\mathcal{R}_{\texttt{Sq}}(t) \leq 8 R^2 L^2 \log(M) \left(G^2 + \max_{i\in[M]}\left\{\lambda_i S^2 + d \log\left({1 + \frac{tL^2 }{\lambda_i d}} \right)\right\} + \log(1/\delta) \right).
\end{align*}


\subsection{Proof of Lemma~\ref{lemm:bound-regret:expectation-counterparts}} \label{proofoflimmaforexpectioncounterpart}

The inequality can be obtained using Azuma's inequality and following similar steps as in Lemma~2 of~\cite{FR-2020}. We may write the regret as 
\begin{align}
\mathcal{R}_{\text{FS-SCB}}(T) &= \sum_{t=1}^T \left(\langle \phi^{i_*}(x_t,a^*_{t}) ,\theta_*^{i_*}\rangle - \langle \phi^{i_*}(x_t,a_t),\theta_*^{i_*}\rangle  - \frac{\alpha}{4}\left( \widehat{y}_t(x_t,a_t) - \langle \phi^{i_*} (x_t,a_t), \theta^{i_*}_*\rangle \right)^2 \right) \nonumber \\ 
&+ \frac{\alpha}{4} \sum_{t=1}^T \left( \widehat{y}_t(x_t,a_t) - \langle \phi^{i_*} (x_t,a_t), \theta^{i_*}_*\rangle \right)^2. 
\label{boundingtheregretowothitscointertpat}
\end{align}
The last term on the RHS of~\eqref{boundingtheregretowothitscointertpat} is bounded with $D_t(\delta)$ in~\eqref{definigitonofd_t:fs} from the result of Lemma~\ref{pred:error:sqalg:FS-SquareCB} in Appendix~\ref{proofoflemaofagentpredictionforsquarecb}.
Define filtration $F_{t-1}= \sigma\big((x_1,a_1,y_1),\dots,(x_{t-1},a_{t-1},y_{t-1})\big)$. On the RHS of~\eqref{boundingtheregretowothitscointertpat}, the action $a_t$ is random. We can use the Azuma's inequality and with probability at least $1-\delta$, upper-bound the first term on the RHS of~\eqref{boundingtheregretowothitscointertpat} with its expectation counterparts using the probability distribution $p_t$ as
\begin{align}
\mathcal{R}_{\text{FS-SCB}}(T) &\leq \sqrt{2T\log(2/\delta)} + \frac{\alpha}{4} D_T \\ 
&+ \sum_{t=1}^T \sum_{a \in [K]} p_t(a)\left(\langle \phi^{i_*}(x_t,a^*_{t}) ,\theta_*^{i_*}\rangle - \langle \phi^{i_*}(x_t,a),\theta_*^{i_*}\rangle  - \frac{\alpha}{4}\left( \widehat{y}_t(x_t,a) - \langle \phi^{i_*} (x_t,a_t), \theta^{i_*}_*\rangle \right)^2 \right). \nonumber
\label{usinglemma2ofraklinformodify}
\end{align}


\subsection{Proof of Theorem~\ref{thm:regretbound_EXP-SquareCB}}
\label{puttinthingstogethr-boundregret-fs}

We first state the following lemma from~\cite{FR-2020} to bound the first term on the RHS of~\eqref{usinglemma2ofraklinformodify}.

\begin{lemma}[Lemma~3 in~\cite{FR-2020}]
\label{lemma3offosterrakhlin}
Under Assumption~\ref{ass:boundedness-Setting2}, for the probability distribution $p_t \in \Delta_K$ defined in the FS-SCB algorithm, we may write
\begin{align*}
\sum_{a \in [K]} p_t(a)\left(\langle \phi^{i_*}(x_t,a^*_{t}) ,\theta_*^{i_*}\rangle - \langle \phi^{i_*}(x_t,a),\theta_*^{i_*}\rangle  - \frac{\alpha}{4}\left( \widehat{y}_t(x_t,a) - \langle \phi^{i_*} (x_t,a_t), \theta^{i_*}_*\rangle \right)^2 \right) \leq \frac{2K}{\alpha}.
\end{align*}
\end{lemma}

Putting everything together, with the choice of $\alpha = \sqrt{KT/D_T(\delta)}$, with probability at least $1-\delta$, we can show the following upper-bound on the regret of the FS-SCB algorithm:
\begin{align}
\mathcal{R}_{\text{\em FS-SCB}}(T) \leq 3 \sqrt{KT D_T(\delta)} + \sqrt{2T \log(2/\delta)}
\end{align}
Here the upper-bound is of order 
\begin{align*}
\mathcal R_{\text{\em FS-SCB}}(T) \leq \mathcal{O}\bigg( \sqrt{2 T \log(2/\delta)} + R L G \sqrt{KT (1+\log(M)) \max_{i \in [M]}\left\{\lambda_i S^2 + 4d \log\left(\frac{1 + \frac{T L^2}{\lambda_i d}}{\delta}\right)\right\}}\bigg). \nonumber
\end{align*}


\newpage
\section{Proofs of Section~\ref{sec:param-select-alg}}
\label{app:proofs-param-selection}

The regret analysis of the PS-OFUL algorithms requires two steps. First, in Theorem~\ref{thm:PS-OFUL-confidence-set}, we show that the confidence set $\mathcal{C}_t$ is valid at each round $t$, i.e., for any $t,\delta>0$, it includes the reward parameter $\theta_*$ with probability at least $1-\delta$. Second, we show how the regret is related to the valid confidence set, and then using Lemmas~\ref{aux:lemm:bound-sumofnorms} and~\ref{aux:lemm:det-teace-ineq} complete the proof. 

\textbf{Step 1.} 
The key idea for showing the validity of the confidence set $\mathcal{C}_t$ requires controlling the square prediction error of the online regression oracle $\widehat{y}_t$, i.e.,~upper-bounding $\gamma_t$. In Appendix~\ref{app:ptoof_themr_confidenceregion_firstsetting}, we show that we can relate this distance to the sum of two terms: $\gamma_t \leq \mathcal{O}\left(U_t + \mathcal{R}_{\texttt{Sq}}(t)\right)$, and then show how we can bound each of them.

{1) Bounding $U_t$:} Lemma~\ref{lem:PS-OFUL-U_t} shows the worst-case upper-bound on the square error of the prediction of true model $i_*$, given that the agent does not know the identity of the true model: 
\begin{align}
 \sum_{s=1}^t \langle \phi_s(a_s), \widehat{\theta}^{i_*}_s &- \theta_* \rangle^2 \leq U_t   
\end{align} 
where 
\begin{align}
    U_t \leq 3 + 8 d \log\left(1 + \frac{t L^2 \max_{i \in [M]} (b_i + c_i)^2}{d} \right) + 32 R^2 \log(1/\delta)
\end{align}

\begin{proof}
The proof is provided in Appendix~\ref{app:subsec:proofoflemaofboudingthexperts}. 
\end{proof}  

{2) Bounding $\mathcal{R}_{\texttt{Sq}}(t)$:} In Lemma~\ref{lem:PS-OFUL-SqAlg-reg}, we prove an upper-bound on the regret caused by the prediction oracle \texttt{SqAlg}, given our proposed expert predictions as (see Appendix~\ref{app:thm:upperbound-regret-oracle} for details).
\begin{align}
    \mathcal{R}_{\texttt{Sq}}(t) \leq 8(G + L)^2\log(M) + 8R^2L^2d\log(M)\log(1/\delta) + 8R^2L^2d\log(M)\log\left(1 + \frac{tL^2\max_{i\in[M]}(b_i+c_i)^2}{d}\right) \nonumber
\end{align}
Putting these together, in Appendix~\ref{app:ptoof_themr_confidenceregion_firstsetting}, we prove Theorem~\ref{thm:PS-OFUL-confidence-set} that shows the validity of the confidence set $\mathcal{C}_t$. 


\textbf{Step 2.}
In Appendix~\ref{finalsecforproofofregrettheorm-psoful}, we first show how regret is related to the confidence set. In particular, we show that given the validity of of the confidence set $\mathcal{C}_t$, i.e., for any $\delta \in (0,1/4]$, with probability at least $1-\delta$, $\theta_* \in \mathcal{C}_t$, we can bound the regret as 
\begin{align*}
    \mathcal{R}_{\text{PS-OFUL}}(T) \leq  2 Gd + 2\max\{1,G\}\sqrt{2 d T \log\left(1 + \frac{T}{d} \right)\max_{d<t\leq T} {\gamma_t(\delta)}}. 
\end{align*}
Then, in Appendix~\ref{app:proof:thm:regret-ps-oful}, we set $\lambda_i = \frac{1}{ (b_i + c_i)^2}$, for each $i\in[M]$, and use Lemmas~\ref{aux:lemm:bound-sumofnorms} and~\ref{aux:lemm:det-teace-ineq} to complete the proof of Theorem~\ref{thm:PS-OFUL-regret}. Here we prove a regret upper-bound  of order 
\begin{align*}
    \mathcal{O}\left(dRL\max\{1,G\}\sqrt{1 + \log(M)} \times \sqrt{T\log\left(1 + \frac{T}{d}\right)\log\left(\frac{1 + \frac{T L^2 \max_{i\in[M]}(b_{i} + c_{i})^2}{ d}}{\delta}\right)}\right).
\end{align*}
%



\subsection{Proof of Lemma~\ref{lem:PS-OFUL-U_t}}
\label{app:subsec:proofoflemaofboudingthexperts}

At each round $s\in[T]$, each expert $i_*\in\mathcal I_*$ estimates its reward parameter as
\begin{align*}
\widehat{\theta}_s^{i_*} = \arg\min_{\theta}  \norm{\Phi_s^\top \theta - Y_s}^2 + \lambda_{i_*} \norm{\theta - \widehat{\mu}_{i_*}}^2,
\end{align*} 
which is the output of a Follow-The-Regularized-Leader (FTRL) algorithm with quadratic regularizer $\norm{\theta - \widehat{\mu}_{i_*}}^2$. Following the standard FTRL analysis of online regression (see e.g.,~\citealt[Chapter 11]{CBL-2006}), we have 
\begin{equation}
\label{eq:LS1}
\sum_{s=1}^t (\langle \phi_s(a_s), \widehat{\theta}^{i_*}_s  \rangle- y_s)^2 - \sum_{s=1}^t ( \langle \phi_s(a_s), \theta_* \rangle- y_s)^2 \leq \lambda_{i_*} \norm{\theta_* - \widehat{\mu}_{i_*}}^2 + 2\sum_{s=1}^t \langle \phi_s(a_s), (V_s^{\lambda_{i_*}})^{-1} \phi_s(a_s) \rangle,
\end{equation}
where $V_t^{\lambda_{i_*}} = \lambda_{i_*} I +  \sum_{s=1}^{t-1} \phi_s(a_s) 
{\phi_s(a_s)}^\top$. We may write~\eqref{eq:LS1} as
\begin{equation}
\label{eq:LS2}
\sum_{s=1}^t \langle \phi_s(a_s), \widehat{\theta}^{i_*}_s  - \theta_* \rangle^2  \leq  \lambda_{i_*} \norm{\theta_* - \widehat{\mu}_{i_*}}^2 + 2 \log\left(\frac{\det(V_t^{\lambda_{i_*}})}{\det(\lambda_{i_*} I)}\right) + 2 \sum_{s=1}^t \eta_s\langle \phi_s(a_s), \widehat{\theta}^{i_*}_s - \theta_* \rangle.
\end{equation}
Using Proposition~\ref{aux:prop:sum-marting-bound} in Appendix~\ref{app:auxily-tools}, we may bound the last term on the RHS of~\eqref{eq:LS2} as 
\begin{equation}
\label{eq:temppp1}
\left| \sum_{s=1}^t \eta_s\langle \phi_s(a_s), \widehat{\theta}^{i_*}_s - \theta_* \rangle \right| \leq R \sqrt{2 \left(1 + \sum_{s=1}^t \langle \phi_s(a_s), \widehat{\theta}^{i_*}_s - \theta_* \rangle^2\right) \log \left(\frac{1 + \sum_{s=1}^t \langle \phi_s(a_s), \widehat{\theta}^{i_*}_s - \theta_* \rangle^2}{\delta} \right)}.
\end{equation}
It is easy to see that~\eqref{eq:LS2} can be written in the form $u^2\leq v + uw$, where $u = \sqrt{1 + \sum_{s=1}^t \langle \phi_s(a_s), \widehat{\theta}^{i_*}_s - \theta_*\rangle^2}$, $v = 1 + \lambda_i\|\theta_* - \widehat{\mu}_{i_*}\|^2 + 2 \log\left(\frac{\det(V_t^{\lambda_{i_*}})}{\det(\lambda_{i_*} I)}\right)$, and $w = 2R\sqrt{2\log(u/\delta)}$. Then, by applying Lemma~\ref{aux:lemma:square-rrot-trick} in Appendix~\ref{app:auxily-tools}, we may write $u \leq \sqrt{v} + w$. Substituting for $w$, we can get $u \leq \sqrt{v} + 2R\sqrt{2\log(u/\delta)}$. Then, by Lemma~\ref{aux:lemma:log-trick} in Appendix~\ref{app:auxily-tools}, for $\delta \in (0,1/4]$, we have
\begin{align*}
u \leq \sqrt{v} + 4 R \sqrt{\log \left(\frac{2\sqrt{2}R + \sqrt{v}}{\delta} \right)},
\end{align*}
which using the inequality $(a + b)^2 \leq 2a^2 + 2b^2$, for any $a$ and $b$, we can write it as 
\begin{align*}
u^2 \leq 2v + 32 R^2 \log \left(\frac{2\sqrt{2}R + \sqrt{v}}{\delta} \right). 
\end{align*} 
Finally, we substitute $u$ and $v$, and subtract $1$ from both sides, and for $\delta \in (0,1/4]$, we obtain 
\begin{equation}
\label{eq:temppp2}
\begin{split}
\sum_{s=1}^t \langle \phi_s(a_s), \widehat{\theta}^{i_*}_s &- \theta_* \rangle^2  \leq 1 + 2\lambda_{i_*} \norm{\theta_* - \widehat{\mu}_{i_*}}^2 + 4\log\left({\frac{\det(V_t^{\lambda_{i_*}})}{\det(\lambda_{i_*} I)}}\right) \\ 
&+ 32 R^2 \log \left(\frac{2\sqrt{2}R + \sqrt{1+\lambda_{i_*} \norm{\theta_* - \widehat{\mu}_{i_*}}^2 + 2\log\left(\frac{\det(V_t^{\lambda_{i_*}})}{\det(\lambda_{i_*} I)}\right) }}{\delta} \right).
\end{split}
\end{equation}
We know $\norm{\theta_* - \widehat{\mu}_{i_*}}^2 \leq (b_{i_*} +c_{i_*})^2$. Moreover, by Lemma~\ref{aux:lemm:det-teace-ineq} in Appendix~\ref{app:auxily-tools}, we can bound the term $\log\left({\frac{\det(V_t^{\lambda_{i_*}})}{\det(\lambda_{i_*} I)}}\right)$. Replacing these terms in~\eqref{eq:temppp2}, we have
\begin{equation}
\label{eq:Ut-Def0}
\begin{split}
\sum_{s=1}^t \langle \phi_s(a_s), \widehat{\theta}^{i_*}_s &- \theta_* \rangle^2 \leq 1 + 2\lambda_{i_*}(b_{i_*} + c_{i_*})^2  + 8d \log\left(1 +  \frac{t L^2}{d\lambda_{i_*}} \right) \\ 
&+ 32 R^2 \log\left(\frac{2\sqrt{2}R + \sqrt{1 + \lambda_{i_*}(b_{i_*} + c_{i_*})^2 + 4d \log(1+ \frac{t L^2}{d\lambda_{i_*}})}}{\delta} \right).
\end{split}
\end{equation}
Setting $\lambda_{i_*}=\frac{1}{(b_{i_*}+c_{i_*})^2}$, as used by the PS-OFUL algorithm, we obtain
\begin{equation}
\label{eq:Ut-Def1}
\begin{split}
\sum_{s=1}^t \langle \phi_s(a_s), \widehat{\theta}^{i_*}_s &- \theta_* \rangle^2 \leq 3 + 8d \log\left(1 + \frac{tL^2(b_{i_*}+c_{i_*})^2}{d} \right) \\ 
&+ 32 R^2 \log\left(\frac{2\sqrt{2}R + \sqrt{2 + 4d \log\big(1+ \frac{tL^2(b_{i_*}+c_{i_*})^2}{d}\big)}}{\delta} \right).
\end{split}
\end{equation}
Since the algorithm does not know the identity of $i_*$, we derive an expression for $U_t$ and conclude the proof by replacing $i_*$ with the maximum over all $i\in[M]$ in~\eqref{eq:Ut-Def1}, as 
\begin{equation}
\label{eq:Ut-Def}
\begin{split}
\sum_{s=1}^t \langle \phi_s(a_s), \widehat{\theta}^{i_*}_s &- \theta_* \rangle^2 \leq 3 + 8d \log\left(1 + \frac{tL^2\max_{i\in[M]}(b_i+c_i)^2}{d} \right) \\ 
&+ 32 R^2 \log\left(\frac{2\sqrt{2}R + \sqrt{2 + 4d \log\big(1+ \frac{tL^2\max_{i\in[M]}(b_i+c_i)^2}{d}\big)}}{\delta} \right) := U_t.
\end{split}
\end{equation}


\subsection{Proof of Lemma~\ref{lem:PS-OFUL-SqAlg-reg}}
\label{app:thm:upperbound-regret-oracle}

To obtain a high probability bound on the regret $\mathcal R_{\texttt{Sq}}(t)$ of the regression oracle \texttt{SqAlg}, we first show that the inputs to the regression oracle, i.e.,~reward signals $y_t=\phi_t(a_t)+\eta_t$ and the experts' predictions $f^i_t(H_t)=\langle \phi_t(a_t), \widehat{\theta}_t^i \rangle$ are all bounded with high probability. We then use Proposition~\ref{app:prop:husslerreulst} in Appendix~\ref{app:SqAlg-description} to complete the proof. 

Since each noise $\eta_t$ is $R$-sub-Gaussian, from Lemma~\ref{aux:lemm:bound-subGaussian-variable} in Appendix~\ref{app:auxily-tools}, with probability at least $1-\delta$, we have that $|\eta_t| \leq R\sqrt{ 2 \log(2/\delta)}$. We also have from Assumption~\ref{ass:boundedness-Setting1} that for each context and each action $a \in \bigcup_{t=1}^T\mathcal{A}_t$, their mean reward $|\langle\phi_t(a),\theta_* \rangle| \leq G$. Thus, by the triangular inequality, with probability at least $1-\delta$, we obtain
\begin{align}
\label{eq:reward-bounded}
y_t \in \left[-\left(G + R\sqrt{2\log(2/\delta)} \right), \left(G + R\sqrt{2\log(2/\delta)}\right)\right].
\end{align}
Next we bound the prediction of the experts that PS-OFUL considers in its prediction. To do so, we employ the same idea as we mentioned in the proof of Lemma~\ref{lem:upper-bound:Rsq(t)} in Appendix~\ref{app:proof:lemm:sqalg:feature-selection}, where we first show an upper bound on the prediction of the any true model $i_*$. In particular, we can write for any time $t \in [T]$: 
\begin{align}
\label{eq:temp00}
\left| \langle \phi_t(a_t), \widehat{\theta}_t^{i_*} \rangle \right| &= \left| \langle \phi_t(a_t), \theta_* \rangle +\langle \phi_t(a_t), \widehat{\theta}_t^{i_*} - \theta_* \rangle \right| \nonumber \\
&\stackrel{\text{(a)}}{\leq} \left|\langle \phi_t(a_t), \theta_* \rangle \right| + \left| \langle \phi_t(a_t), \widehat{\theta}_t^{i_*} - \theta_* \rangle \right| \nonumber \\
&\stackrel{\text{(b)}}{\leq} G + \norm{\phi_t(a_t)}_{(V_t^{\lambda_i})^{-1}} \left(\norm{\Phi_t \eta_t}_{(V_t^{\lambda_i})^{-1}} + \sqrt{\lambda_{i_*}} \norm{\widehat{\mu}_{i_*} - \theta_*} \right) \nonumber \\& \stackrel{\text{(c)}}{\leq}
G + RL\sqrt{d \log\left(\frac{1 + \frac{t L^2}{\lambda_i d}}{\delta} \right)} + L\sqrt{\lambda_i} (b_i + c_i)
\end{align}

{\bf (a)} It results from triangular inequality. {\bf (b)} This comes from the Assumption~\ref{ass:boundedness-Setting1} as well as Theorem~1 in \citet{abbasi2011improved}. {\bf (c)} This is because of the  Theorem~2 in \citet{abbasi2011improved} and the fact that $i_*$ is the true model and hence $\theta_* \in B(\widehat{\mu}_{i_*},b_{i_*})$. Thus, we can have $\norm{\widehat{\mu}_{i_*} - \theta_*} \leq (b_i + c_i)$. PS-OFUL employees this idea that at any time step, any potentially true model should have a similar bound on its prediction. This is being enforced by the set of admissible expert, $\mathcal{S}_t$, where it only considers experts that have the following bound on their prediction at each time $t \in [T]$ as:
\begin{align}
\label{eq:temp04}
\left| \langle \phi_t(a_t), \widehat{\theta}_t^i \rangle \right| \leq G + RL\sqrt{d \log\left(\frac{1 + \frac{t L^2}{\lambda_i d}}{\delta} \right)} + L\sqrt{\lambda_i} (b_i + c_i).
\end{align}
If at some time step $t$, this bound does not hold for any expert $i$, then the algorithm simply eliminates that expert from the set of admissible experts, since that model is not a true model (i.e., the reward does not belong to the ball of that model), and that expert will remain out for the rest of the game.

%
Setting $\lambda_i = \frac{1}{(b_i+c_i)^2}$ in~\eqref{eq:temp04}, we can bound the prediction of each expert $i\in \mathcal{S}_t$ at round $t\in[T]$ as 
\begin{align}
\label{eq:temp05}
\langle \phi_t(a_t), \widehat{\theta}_t^i \rangle \in
\bigg[-\bigg(G + L + &RL\sqrt{d \log\Big(\frac{1 + \frac{tL^2\max_{i \in [M]} (b_i+c_i)^2}{d}}{\delta}\Big)}\bigg) \nonumber \\ 
&,\bigg(G + L + RL\sqrt{d \log\Big(\frac{1 + \frac{tL^2\max_{i \in [M]} (b_i+c_i)^2}{d}}{\delta}\Big)}\bigg)\bigg].
\end{align}
Putting together~\eqref{eq:reward-bounded} and~\eqref{eq:temp05}, we conclude that for all rounds $t\in[T]$ and experts $i\in\mathcal{S}_T$, with probability at least $1-\delta$, the rewards $y_t$ and the experts' predictions $f^i_t(H_t)$ are in the range $[\beta, \beta+\ell]$ for 
\begin{equation}
\label{eq:temp06}
\begin{split}
\beta &= -\bigg(G + L + RL\sqrt{d \log\Big(\frac{1 + \frac{t L^2\max_{i\in[M]}(b_i+c_i)^2}{d}}{\delta} \Big)} \bigg), \\
\ell &= 2\bigg(G + L + RL\sqrt{d \log\Big(\frac{1 + \frac{t L^2\max_{i \in [M]}(b_i+c_i)^2}{d}}{\delta}\Big)}\bigg).
\end{split}
\end{equation} 
Using Proposition~\ref{app:prop:husslerreulst} in Appendix~\ref{app:SqAlg-description} with the bound on the observations and predictions in~\eqref{eq:temp06}, with probability at least $1-\delta$, we obtain the following regret bound for \texttt{SqAlg}: 
\begin{align}
\label{eq:temp07}
\mathcal{R}_{\texttt{Sq}}(t) \leq 8(\log M) \bigg( (G+ L)^2 + R^2 L^2 d \log\Big(\frac{1 + \frac{tL^2\max_{i\in[M]}(b_i+c_i)^2}{d}}{\delta} \Big) \bigg),
\end{align} 
in which we use the fact that for $a,b>0,\;(a + b)^2 \leq 2 a^2 + 2b^2$. This concludes our proof.


\subsection{Proof of Theorem~\ref{thm:PS-OFUL-confidence-set}}
\label{app:ptoof_themr_confidenceregion_firstsetting}

In order to fully specify the confidence set $\mathcal C_t$ and prove its validity, i.e.,~$\theta_*\in\mathbb P(\theta_*\in\mathcal C_t) \geq 1-\delta$, we should find a high probability upper-bound $\gamma_t(\delta)$ for the sum of the square loss of the oracle predictions, i.e., 
\begin{equation*}
\sum_{s=1}^t \left(\widehat{y}_s - \langle \phi_s(a_s), \theta_* \rangle \right)^2 \leq \gamma_t(\delta).
\end{equation*}
Let $z_s = (\widehat{y}_s - y_s)^2 - (\langle \phi_s(a_s), \widehat{\theta}_s^{i_*} \rangle - y_s)^2$, where $i_*\in\mathcal I_*$ is the index of a ball that contains $\theta_*$. Since $y_s = \langle \phi_s(a_s), \theta_* \rangle + \eta_s$, we may write 
\begin{align*}
z_s &= (\widehat{y}_s - \langle \phi_s(a_s), \theta_* \rangle -\eta_s)^2 - (\langle \phi_s(a_s), \widehat{\theta}_s^{i_*} \rangle - \langle \phi_s(a_s), \theta_* \rangle -\eta_s)^2 \\
&= (\widehat{y}_s - \langle \phi_s(a_s), \theta_* \rangle)^2 - (\langle \phi_s(a_s), \widehat{\theta}_s^{i_*} \rangle - \langle \phi_s(a_s), \theta_* \rangle)^2 + 2 \eta_s (\langle \phi_s(a_s), \widehat{\theta}_s^{i_*} \rangle - \widehat{y}_s). 
\end{align*}
Since $\sum_{s=1}^t z_s \leq {\mathcal R}_{\texttt{Sq}}(t)$, where ${\mathcal R}_{\texttt{Sq}}(t)$ is the regret of the regression oracle at round $t$, we have 
\begin{equation}
\begin{split}
\label{first-upperboundonerorofaggregation}
\sum_{s=1}^t (\widehat{y}_s - \langle \phi_s(a_s), \theta_* \rangle)^2 \leq {\mathcal R}_{\texttt{Sq}}(t) + \sum_{s=1}^t (\langle \phi_s(a_s), \widehat{\theta}_s^{i_*} \rangle - \langle \phi_s(a_s), \theta_* \rangle)^2 + 2\sum_{s=1}^t \eta_s (\langle \phi_s(a_s), \widehat{\theta}_s^{i_*} \rangle - \widehat{y}_s).
\end{split}
\end{equation}
From the definition of $U_t$ in~\eqref{eq:Ut}, we may upper-bound $\sum_{s=1}^t (\langle \phi_s(a_s), \widehat{\theta}_s^{i_*} \rangle - \langle \phi_s(a_s), \theta_* \rangle)^2$ with $U_t$ and write~\eqref{first-upperboundonerorofaggregation} as
\begin{equation}
\begin{split}
\label{forcombingresult}
\sum_{s=1}^t (\widehat{y}_s &- \langle \phi_s(a_s), \theta_* \rangle)^2 \leq {\mathcal R}_{\texttt{Sq}}(t) + U_t + 2 \sum_{s=1}^t \eta_s (\langle \phi_s(a_s), \widehat{\theta}_s^{i_*}\rangle - \widehat{y}_s) \\ 
&\leq {\mathcal R}_{\texttt{Sq}}(t) + U_t + 2 \sum_{s=1}^t \eta_s \langle \phi_s(a_s), \widehat{\theta}_s^{i_*} - \theta_* \rangle + 2 \sum_{s=1}^t \eta_s (\langle \phi_s(a_s), \theta_* \rangle - \widehat{y}_s).
\end{split}
\end{equation} 
Then, from Proposition~\ref{aux:prop:sum-marting-bound} in Appendix~\ref{app:auxily-tools}, with probability at least $1-\delta$, we have 
\begin{align}
\label{eq:temp0}
&\left| \sum_{s=1}^t \eta_s \langle \phi_s(a_s),\widehat{\theta}_s^{i_*} - \theta_* \rangle\right| \leq \\
&\qquad\qquad R \sqrt{2 \left(1 + \sum_{s=1}^t\langle \phi_s(a_s), \widehat{\theta}_s^{i_*} - \theta_* \rangle^2 \right) \log \left(\frac{\sqrt{1 + \sum_{s=1}^t \langle \phi_s(a_k), \widehat{\theta}_s^{i_*} - \theta_* \rangle^2}}{\delta}\right)}, \nonumber
\end{align} 
and
\begin{align}
\label{eq:temp1}
&\left| \sum_{s=1}^t \eta_s (\langle \phi_s(a_s), \theta_* \rangle - \widehat{y}_s ) \right| \leq \\
&\qquad\qquad R \sqrt{2 \left(1 + \sum_{s=1}^t \left( \langle \phi_s(a_s), \theta_* \rangle - \widehat{y}_s \right)^2 \right) \log \left( \frac{\sqrt{1 + \sum_{s=1}^t \left(\langle \phi_s(a_s), \theta_* \rangle - \widehat{y}_s\right)^2 }}{\delta} \right)}. \nonumber
\end{align} 
Using~\eqref{eq:temp0} and~\eqref{eq:temp1}, we may write~\eqref{forcombingresult} as 
\begin{equation}
\label{boundthelognfomrofradiusellipsoid}
\begin{split}
\sum_{s=1}^t \big(\widehat{y}_s &- \langle \phi_s(a_s), \theta_* \rangle\big)^2 \leq  {\mathcal R}_{\texttt{Sq}}(t)+ U_t + 2 R\sqrt{2 (1 + U_t) \log\big(\sqrt{1+U_t}/\delta\big) } \\ 
&+ R\sqrt{8 \left(1 + \sum_{s=1}^t \big(\widehat{y}_s - \langle \phi^s(a_s), \theta_* \rangle\big)^2  \right) \log \left( \frac{\sqrt{1 + \sum_{s=1}^t \big(\widehat{y}_s - \langle \phi_s(a_s), \theta_* \rangle\big)^2}}{\delta} \right)}. 
\end{split}
\end{equation}
It is easy to see that~\eqref{boundthelognfomrofradiusellipsoid} can be written in the form $u^2\leq v + uw$, where $u = \sqrt{1 + \sum_{s=1}^t (\langle \phi_s(a_s), \theta_* \rangle - \widehat{y}_s)^2}$, $v = 1 + {\mathcal R}_{\texttt{Sq}}(t) + U_t + 2R\sqrt{2 (1+ U_t) \log(\frac{\sqrt{1+U_t}}{\delta}) }$, and $w = R\sqrt{8\log(u/\delta)}$. Then, by applying Lemma~\ref{aux:lemma:square-rrot-trick} in Appendix~\ref{app:auxily-tools}, we may write $u \leq w + \sqrt{v}$. Substituting for $w$, we can get $u \leq \sqrt{v} + R\sqrt{8\log(u/\delta)}$. Then, by Lemma~\ref{aux:lemma:log-trick} in Appendix~\ref{app:auxily-tools}, for $\delta \in (0,1/4]$, we have 
\begin{align*}
u \leq \sqrt{v} + 4 R \sqrt{\log \left(\frac{R\sqrt{8} + \sqrt{v}}{\delta} \right)},
\end{align*}
which using the inequality $(a + b)^2 \leq 2a^2 + 2b^2$, for any $a$ and $b$, we can write it as 
\begin{align*}
u^2 \leq 2 v^2 + 32 R^2 \log \left(\frac{R\sqrt{8} + \sqrt{v}}{\delta} \right). 
\end{align*} 
Finally, we substitute $u$ and $v$, and subtract $1$ from both sides, and for $\delta \in (0,1/4]$, we obtain 
\begin{equation}
\label{eq:final-conf-set0}
\begin{split}
\sum_{s=1}^t \big(\widehat{y}_s &- \langle \phi_s(a_s), \theta_* \rangle\big)^2 \leq  1 + 2 {\mathcal R}_{\texttt{Sq}}(t) + 2 U_t + 4R\sqrt{2 (1+ U_t) \log\big(\sqrt{1+U_t}/\delta}\big) \\ 
&+32 R^2 \log\left(\frac{R\sqrt{8} + \sqrt{1 + {\mathcal R}_{\texttt{Sq}}(t) + U_t + 2R\sqrt{2 (1+ U_t) \log\big(\sqrt{1+U_t}/\delta\big)}}}{\delta}\right) := \gamma_t(\delta).
\end{split}
\end{equation}
Eq.~\ref{eq:final-conf-set0} shows that for $\delta \in (0,1/4]$, with probability at least $1-\delta$, we have $\theta^*\in C_t$, which completes the proof of the validity of the confidence set $\mathcal C_t$. 

We can now fully specify $\mathcal C_t$ by plugging $U_t$ from~\eqref{eq:Ut-Def} (see Appendix~\ref{app:subsec:proofoflemaofboudingthexperts}) and ${\mathcal R}_{\texttt{Sq}}(t)$ from~\eqref{eq:temp07} (see Appendix~\ref{app:thm:upperbound-regret-oracle}) into~\eqref{eq:final-conf-set0}, and write $\gamma_t(\delta)$ as 
\begin{align}
\label{eq:final-conf-set}
\gamma_t(\delta) &:= 1 + 2 {\mathcal R}_{\texttt{Sq}}(t) + 2 U_t + 4R\sqrt{2 (1+ U_t) \log\big(\sqrt{1+U_t}/\delta\big)} \nonumber \\ 
&\qquad\qquad +32 R^2 \log\left( \frac{R\sqrt{8} + \sqrt{1 + {\mathcal R}_{\texttt{Sq}}(t) + U_t + 2R\sqrt{2 (1+ U_t) \log\big(\sqrt{1+U_t}/\delta\big)}}}{\delta}\right), \nonumber \\
&\hspace{-0.95in}\text{where} \\
U_t &= 3 + 8d \log\left(1 + \frac{tL^2\max_{i\in[M]}(b_i+c_i)^2}{d} \right) \nonumber \\ 
&\qquad\qquad + 32 R^2 \log\left(\frac{2\sqrt{2}R + \sqrt{2 + 4d \log\left(1+ \frac{tL^2\max_{i\in[M]}(b_i+c_i)^2}{d}\right)}}{\delta} \right), \nonumber \\
\mathcal{R}_{\texttt{Sq}}(t) &= 8\log(M) \left(G^2 + L^2 + {2GL} + R^2 L^2 d \log\left(\frac{1 + \frac{tL^2\max_{i\in[M]}(b_i+c_i)^2}{d}}{\delta} \right) \right),
\nonumber
\end{align}
which concludes the proof.


A closer look at $U_t$ and ${\mathcal R}_{\texttt{Sq}}(t)$, the two main terms in the definition of $\gamma_t(\delta)$, we may write them in terms of the dominant terms as 
\begin{align}
\label{eq:Ut-simplified}
U_t &\approx \overbrace{3 + 16R^2\log(2)}^{C_1} + 8d \log\left(1 + \frac{tL^2\max_{i\in[M]}(b_i+c_i)^2}{d} \right) + 32 R^2 \log(1/\delta) \nonumber \\ 
&\qquad\qquad\qquad\qquad + 32R^2\log\left(1 + 2R + d\log\left(1 + \frac{tL^2\max_{i\in[M]}(b_i+c_i)^2}{d} \right)\right) \nonumber \\
&\approx C_1 + 32 R^2 \log(1/\delta) + 8d \log\left(1 + \frac{tL^2\max_{i\in[M]}(b_i+c_i)^2}{d} \right),
\end{align}
and 
\begin{align}
\label{eq:Rsq-simplified}
\mathcal{R}_{\texttt{Sq}}(t) &= \overbrace{8(G + L)^2\log(M)}^{C_2} + 8R^2L^2d\log(M)\log(1/\delta) + 8R^2L^2d\log(M)\log\left(1 + \frac{tL^2\max_{i\in[M]}(b_i+c_i)^2}{d} \right) \nonumber \\
&= C_2 + 8R^2L^2d\log(M)\log(1/\delta) + 8R^2L^2d\log(M)\log\left(1 + \frac{tL^2\max_{i\in[M]}(b_i+c_i)^2}{d} \right).
\end{align}
Using~\eqref{eq:Ut-simplified} and~\eqref{eq:Rsq-simplified}, we may write $\gamma_t(\delta)$ in terms of the dominant terms as 
\begin{align}
\label{eq:gamma-simplified}
\gamma_t(\delta) \approx 1 + 2C_1 + 2C_2 &+ 16R^2\left(4 + L^2d\log(M)\right)\log(1/\delta) \nonumber \\ 
&+ 16d\left(1 + R^2L^2\log(M)\right)\log\left(1 + \frac{tL^2\max_{i\in[M]}(b_i+c_i)^2}{d} \right).
\end{align}
%


\subsection{Proof of Lemma~\ref{lem:PS-OFUL-reg:sumup}} \label{finalsecforproofofregrettheorm-psoful}

In Theorem~\ref{thm:PS-OFUL-confidence-set}, we proved that at each round, with probability at least $1-\delta$, the true reward parameter $\theta_*$ belongs to the confidence set $\mathcal{C}_t$ of the PS-OFUL algorithm. Here, we show how the regret of PS-OFUL is related to the radius $\gamma_t(\delta)$ of this confidence set. 

Here we assume that at the first $d$ rounds, the algorithm plays actions whose features are of the form $\phi_i(a_i) = L e_i,\;\forall i \in [d]$, where $e_i=[0,\dots,1,\dots,0]$ is a $d$-dimensional vector whose elements are all $0$, except a $1$ at the $i^{\text{th}}$ position. In this case, we can define a matrix $V_t$ as
\begin{align}
V_t  &= \sum_{s=1}^{t-1} \phi_s(a_s)^\top \phi_s(a_s) = L^2 I + \sum_{s=d+1}^{t-1} \phi_t(a_t)^\top \phi_t(a_t),
\label{fictionalcovariancemateix}
\end{align} 
and use it to rewrite the confidence set as
\begin{align}
\label{centertforefictioanlellipdoide}
\mathcal{C}_{t-1} = \big\{\theta \in \mathbb{R}^d : (\theta - \widehat{\theta}_t) V_t (\theta - \widehat{\theta}_t) + \sum_{s=1}^{t-1} \big(\widehat{y}_s - \langle \phi_s(a_s), \widehat{\theta}_t \rangle \big)^2 \leq \gamma_t(\delta) \big\}, 
\end{align}
where $\widehat{\theta}_t = \argmin_{\theta \in \mathbb{R}^d} \; \sum_{s=1}^{t-1} \big(\widehat{y}_s - \langle \phi_s(a_s), \theta \rangle \big)^2$. The confidence set $\mathcal{C}_t$ in~\eqref{centertforefictioanlellipdoide} is contained in a larger ellipsoid
\begin{align}
\mathcal{C}_{t-1} \subseteq \big\{\theta \in \mathbb{R}^d : (\theta - \widehat{\theta}_t) V_t (\theta - \widehat{\theta}_t) \leq \gamma_t(\delta)\big\} = \big\{\theta \in \mathbb{R}^d :  \|\theta - \widehat{\theta}_t\|_{V_t}^2  \leq \gamma_t(\delta)\big\}.
\end{align}
Given $(a_t,\widetilde{\theta}_t)=\arg\max_{a\in \mathcal{A}_t}\max_{\theta\in\mathcal{C}_{t-1}} \langle\phi_t(a),\theta\rangle$ are the action and parameter resulted from solving the optimization problem at round $t$ of the PS-OFUL algorithm, we may write
\begin{align}
\label{eq:temp00000}
\langle\phi_t(a_t^*), \theta_* \rangle  - \langle \phi_t(a_t) ,\theta_*\rangle &\leq \langle \phi_t(a_t) , \widetilde{\theta}_t \rangle -  \langle \phi_t(a_t) ,\theta_*\rangle \nonumber \\
&= \langle \phi_t(a_t) , \widetilde{\theta}_t - \widehat{\theta}_t \rangle + \langle \phi_t(a_t) ,\widehat{\theta}_t - \theta_* \rangle \nonumber \\ 
&\leq \norm{\phi_t(a_t)}_{V_t^{-1}} \|\widetilde{\theta}_t - \widehat{\theta}_t\|_{V_t} +  \norm{\phi_t(a_t)}_{V_t^{-1}} \|\widehat{\theta}_t - \theta_*\|_{V_t} \nonumber \\ 
&\leq 2 \sqrt{\gamma_t(\delta)} \norm{\phi_t(a_t)}_{V_t^{-1}} ~~~~~~~~~(\text{because}~ \theta_*, \widetilde{\theta}_t \in \mathcal{C}_{t-1}).
\end{align}
Since $\forall a \in \bigcup_{t=1}^T \mathcal{A}_t$, we assume that $|\langle \phi(a), \theta_* \rangle| \leq G$, we can upper-bound the instantaneous regret in~\eqref{eq:temp00000} as 
\begin{align}
\langle\phi_t(a_t^*), \theta_* \rangle  - \langle \phi_t(a_t) ,\theta_*\rangle \leq 2 \min \big\{G, \sqrt{\gamma_t(\delta)} \norm{\phi_t(a_t)}_{V_t^{-1}}\big\}. 
\label{uppernoundoninstantenousregret}
\end{align}
Using \eqref{uppernoundoninstantenousregret}, we can bound the transfer-regret of PS-OFUL as
\begin{align}
\mathcal{R}_{\text{PS-OFUL}}(T) &= \sum_{t=1}^T \langle\phi_t(a_t^*) - \phi_t(a_t), \theta_* \rangle \leq 2Gd+ \sum_{t=d+1}^T \langle\phi_t(a_t^*) - \phi_t(a_t), \theta_* \rangle
\nonumber \\&
\leq 2 Gd + 2 \sum_{t=d+1}^T \min \{G,  \sqrt{\gamma_t(\delta)} \norm{\phi_t(a_t)}_{V_t^{-1}}\}
\nonumber \\&
\leq 2 Gd + 2 \sum_{t=d+1}^T \sqrt{\gamma_t(\delta)} ~ \min \{G,  \norm{\phi_t(a_t)}_{V_t^{-1}}\} ~~~~~~~~~(\text{since}~ \gamma_t(\delta) \geq 1)
\nonumber \\&
\leq 2 Gd + 2 \left(\max_{d<t\leq T} \sqrt{\gamma_t(\delta)} \right) \sum_{t=d+1}^T  \min \{G,  \norm{\phi_t(a_t)}_{V_t^{-1}}\} 
\nonumber \\&
\leq 2 Gd + 2 \left(\max_{d<t\leq T} \sqrt{\gamma_t(\delta)} \right) \left(\max\{1,G\}\right) \sum_{t=d+1}^T \min \{1, \norm{\phi_t(a_t)}_{V_t^{-1}} \}
\nonumber \\&
\leq 2 Gd + 2 \left(\max_{d<t\leq T} \sqrt{\gamma_t(\delta)} \right) \left(\max\{1,G\}\right) ~ \sqrt{ T \sum_{t=d+1}^T \min \{1, \norm{\phi_t(a_t)}_{V_t^{-1}}^2 \}}
\nonumber \\& 
\leq 2 Gd + 2 \left(\max_{d<t\leq T} \sqrt{\gamma_t(\delta)} \right) \left(\max\{1,G\}\right) \sqrt{ 2T \log\left(\frac{\det(V_T)}{\det(V_d)} \right) },\label{boundnghteregrerpsofultransfer}
\end{align} 
where  the last inequality follows from Lemma~\ref{aux:lemm:bound-sumofnorms} in Appendix~\ref{app:auxily-tools}. Then,  using Lemma~\ref{aux:lemm:det-teace-ineq} in Appendix~\ref{app:auxily-tools}, we can bound $\det(V_T) \leq \left( L^2 + \frac{T L^2}{d}\right)^d$ and $\det(V_d) = L^{2d}$. Hence, we may write \eqref{boundnghteregrerpsofultransfer} as
\begin{align}
\mathcal{R}_{\text{PS-OFUL}}(T) \leq  2 Gd + 2\max\{1,G\}\sqrt{2 d T \log\left(1 + \frac{T}{d} \right)\max_{d<t\leq T} {\gamma_t(\delta)}}. 
\label{eq:relating:regert:to:condifence-set}
\end{align}

  
\subsection{Proof of Theorem~\ref{thm:PS-OFUL-regret}}
\label{app:proof:thm:regret-ps-oful}

If we substitute $\gamma_t(\delta)$ from~\eqref{eq:gamma-simplified} in the regret bound~\eqref{eq:relating:regert:to:condifence-set}, we may write it (in terms of the dominant terms) as
\begin{align}
\label{eq:PS-OFUL-regret-simplified}
\mathcal{R}_{\text{PS-OFUL}}(T) &\leq 2 Gd + 2\sqrt{2}\max\{1,G\}\sqrt{dT\log\left(1 + \frac{T}{d}\right)} \nonumber \\ 
&\hspace{-0.25in}\times\sqrt{C_3 + 16R^2\big(4 + L^2d\log(M)\big)\log(1/\delta) + 16d\big(1 + R^2L^2\log(M)\big)\log\left(1 + \frac{TL^2\max_{i\in[M]}(b_i + c_i)^2}{d}\right)} \nonumber \\
&= \mathcal{O}\left(dRL\max\{1,G\}\sqrt{1 + \log(M)} \times \sqrt{T\log\left(1 + \frac{T}{d}\right)\log\left(\frac{1 + \frac{T L^2 \max_{i\in[M]}(b_{i} + c_{i})^2}{ d}}{\delta}\right)}\right),
\end{align}
where $C_3 = 1 + 2C_1 + 2C_2$, and hence $C_3 = 7 + 32 R^2 \log(2) + 16 (G+L)^2 \log(M)$.

\newpage
\section{Auxiliary Tools}
\label{app:auxily-tools}

Here we report auxiliary results that we use in our proofs in other appendices. 

We start with stating Theorem~7 in~\cite{APS-2012}, which is the self-normalized martingale tail inequality for the scalar random variables.  

\begin{proposition}[Self-normalized bound for martingales] \label{aux:prop:sum-marting-bound}
Let $\{F_t\}_{t=1}^\infty$ be a filtration. Let $\tau$ be a stopping time w.r.t to the filtration $\{F_t\}_{t=1}^\infty$, i.e., the event $\{\tau \leq t \}$ belongs to $F_{t+1}$. Let $\{Z_t\}_{t=1}^\infty$ be a sequence of real-valued variables such that $Z_t$ is $F_t$-measurable. Let $\{\eta_t\}_{t=1}^\infty$ be a sequence of real-valued random variables such that $\eta_t$ is $F_{t+1}$ measurable and is conditionally $R$-sub-Gaussian. Then, for any $\delta>0$, with probability at least $1-\delta$, 
\begin{align*}
    \left\|\sum_{t=1}^\tau \eta_t Z_t\right\| \leq R\sqrt{2 \left(1 + \sum_{t=1}^\tau Z_t^2 \right) \log\left(\frac{\sqrt{1 + \sum_{t=1}^\tau Z_t^2 }}{\delta} \right) }.
\end{align*}
\end{proposition}

Next, we state a direct application of Lemma~11 in~\cite{abbasi2011improved} that bounds the cumulative sum of $\sum_{s=1}^{t-1}  \norm{\phi_s(a_s)}_{V_s^{-1}}^2$ which plays an important role in most of the proofs for linear bandits problems.

\begin{lemma}
\label{aux:lemm:bound-sumofnorms}
Let $\lambda > 0$ and $V_t = \lambda I + \sum_{s=1}^{t-1}\phi_s(a_s) \phi_s^\top(a_s) $. If for all $a \in \cup_{s=1}^{t-1}\mathcal{A}_s$, we have $\norm{\phi_s(a)}_2 \leq L$, then we may write
\begin{align*}
    \sum_{s=1}^{t-1}  \min\{1, \norm{\phi_s(a_s)}_{V_s^{-1}}^2\} \leq 2 \log\left(\frac{\det(V_t)}{\det(\lambda I)} \right).
\end{align*}
\end{lemma}

Next, we present a determinant-trace inequality matrix result.

\begin{lemma}[Determinant-Trace Inequality] 
\label{aux:lemm:det-teace-ineq}
Suppose $X_1,\dots,X_{t-1} \in \mathbb{R}^d$, and for any $1 \leq s \leq t-1$, we have $\norm{X_s}_2 \leq L$. Let $V_t = \lambda I + \sum_{s=1}^{t-1} X_s X_s^\top$, for some $\lambda > 0$. Then we have 
\begin{align*}
    \det(V_t) \leq \left(\lambda +  \frac{tL^2}{d} \right)^d.
\end{align*} 
\end{lemma}

\begin{proof}
Let $\alpha_1,\dots,\alpha_d$ be the eigenvalues of $V_t$. Since $V_t$ is positive definite, its eigenvalues are positive. Also not that $\det(V_t) = \Pi_{s=1}^{d} \alpha_s$ and ${trace}(V_t) = \sum_{s=1}^d \alpha_s$. By arithmetic-geometric means inequality we have 
\begin{align*}
    \sqrt[d]{\alpha_1 \dots \alpha_d} \leq \frac{\alpha_1 + \dots + \alpha_d}{d}.
\end{align*} 
Therefore, $\det(V_t) \leq \left( \frac{trace(V_t)}{d} \right)^d$. It suffices to upper-bound the trace of $V_t$ as 
\begin{align*}
   trace(V_t) = trace(\lambda I) + \sum_{s=1}^{t-1} trace(X_s X_s^\top) = d \lambda + \sum_{s=1}^{t-1} \norm{X_s}^2_2 \leq d \lambda + t L^2, 
\end{align*} 
and the result follows. 
\end{proof}

Next, we state a bound on the absolute value of the $R$-sub-Gaussian random variable. 

\begin{lemma}
\label{aux:lemm:bound-subGaussian-variable}
Let $\{F_t\}_{t=1}^\infty$ be a filtration. Let $\{\eta\}_{t=1}^\infty$ be a real-valued stochastic process such that $\eta_t$ is $F_t$-measurable and $\eta_t$ is conditionally $R$-sub-Gaussian for some $R>0$, i.e., 
\begin{align*}
    \forall \lambda \in \mathbb{R}, ~~\mathbb{E}\left[\eta_t | F_t \right] = 0, ~~ \mathbb{E}\left[ e^{\lambda \eta_t} | F_t\right] \leq \exp{\left(\frac{\lambda^2 R^2}{2} \right)}.
\end{align*} 
Then, condition on filtration $F_t$, with probability at least $1-\delta$, we have $|\eta_t| \leq R\sqrt{ 2 \log(2/\delta)}$.
\end{lemma}

\begin{proof} 
Let $\lambda > 0$. Then, 
\begin{align}
\mathbb{P}(\eta_t \geq k | F_t) &= \mathbb{P}(e^{\lambda \eta_t} \geq e^{\lambda k} | F_t) \leq e^{-\lambda k}~ \mathbb{E}[e^{\lambda \eta_t} | F_t] ~~~~~~~\text{(by Markov's inequality)} \nonumber\\&
     \leq e^{-\lambda k} e^{\frac{\lambda^2 R^2}{2}} = \exp{\left(-\lambda k + \frac{\lambda^2 R^2}{2}  \right)}.
\end{align} 
Optimizing for $\lambda$, and thus, selecting $\lambda = \frac{k}{R^2}$, we conclude that
\begin{align*}
    \mathbb{P}(\eta_t \geq k | F_t) \leq e^{-\frac{k^2}{2 R^2}}.
\end{align*}
Repeating this argument for $-\eta_t$, we also obtain $\mathbb{P}(\eta_t \leq -k | F_t) \leq e^{-\frac{k^2}{2 R^2}}$. Combining these two bounds, we can conclude that 
\begin{align}
    \mathbb{P}(|\eta_t| \geq k | F_t) \leq 2 e^{-\frac{k^2}{2 R^2}}.
    \label{app:aux:lemm:boundonabseloutvalueofeta}
\end{align} 
From~\eqref{app:aux:lemm:boundonabseloutvalueofeta}, with the choice of $\delta = 2 e^{-\frac{k^2}{2 R^2}}$, and thus $k= R\sqrt{2 \log(2/\delta)}$, completes the proof. 
\end{proof}

Then, we state a square-root trick for positive numbers. 

\begin{lemma} 
\label{aux:lemma:square-rrot-trick}
Let $a,b > 0$. If $z^2 \leq a + b z$, then $z \leq \sqrt{a} + b$. 
\end{lemma}

\begin{proof}
Let $q(z) = z^2 - bz -a$. We can rewrite the condition $z^2 \leq a + bz$ as $q(z) \leq 0$. Then we know that the quadratic polynomial $q(z)$ has the following two roots 
\begin{align}
    z^*_{1} = \frac{b + \sqrt{b^2 + 4a}}{2} ~z^*_{2} = \frac{b - \sqrt{b^2 + 4a}}{2}. \nonumber
\end{align} 
Then, we know that the condition $q(z) \leq 0$, implies that $\min\{z^*_1,z^*_2\} \leq z \leq \max\{z^*_1,z^*_2\}$. Therefore, for positive numbers $a, b$, we get 
\begin{align*}
    z \leq  \max\{z^*_1,z^*_2\} = \frac{b + \sqrt{b^2 + 4a}}{2} \leq b + \sqrt{a},
\end{align*} 
where for the last inequality, we use the fact that for $u,v > 0$, $\sqrt{u + v} \leq \sqrt{u} +\sqrt{v}$.
\end{proof}

Next, we restate a simple logarithmic trick from~\cite{APS-2012}.

\begin{lemma}[Proposition 10 in \citealt{APS-2012}]
\label{aux:lemma:log-trick}
Let $c \geq 1$, $q>0$, $\delta \in (0,1/4]$. If $s \geq 1$ and $s \leq c + q \sqrt{\log(s/\delta)}$, then we have $s \leq c + q \sqrt{2 \log(\frac{c + q}{\delta})}$.
\end{lemma}

\newpage
\section{Relation to Latent Bandits}
\label{app:latent bandits}

In this section, we informally show that if the goal in latent bandits is to have a better scaling with the number of actions $K$ (e.g.,~the number of actions $K$ is much larger than the number of latent states $M$), we can use a different bandit model selection strategy, called \textit{regret balancing}~\citep{APP-2020,PDGB-2020, pacchiano2020model,cutkosky2021dynamic} to obtain an improved regret that scales as $\min\{\varepsilon T + \sqrt{MT}, \sqrt{K M T}\}$. This rate is the best of the regret of PS-OFUL, which scales as $\sqrt{KT}$, and the regret of the latent bandit algorithm of~\citet{HKZCAB-2020}, which scales as $\varepsilon T + \sqrt{MT}$.

In regret balancing, in each round, the model selection strategy chooses one of $M$ base algorithms. We denote by $N_{i,t}$, the number of times that the base algorithm $i$ has been selected up to round $t$, and by $R_{i,t}$, the cumulative rewards of this base algorithm during these $N_{i,t}$ rounds. Given a reference regret bound $U:[T]\rightarrow\mR$, in each round $t\in[T]$, the algorithm first finds the optimistic base algorithm $I_t$ and its value $b_t$, i.e.,
\begin{equation}
\label{eq:regret-balancing}
I_t = \argmax_{i\in [M]} \; \frac{R_{i,t}}{N_{i,t}} + \frac{U(N_{i,t})}{N_{i,t}}\qquad,\qquad b_t = \frac{R_{I_t,t}}{N_{I_t,t}} + \frac{U(N_{I_t,t})}{N_{I_t,t}} \;,
\end{equation}
and then takes the action recommended by $I_t$ and uses its observed reward to update the base algorithm $I_t$.

We can apply regret balancing to the problem of latent bandits in the following way. We consider $M+1$ base algorithms: one that plays UCB, and $M$, each corresponds to a latent value and always plays the greedy action of that latent model (which is guaranteed to be $\varepsilon$-accurate by assumption). If the regret balancing strategy selects the UCB base algorithm in all rounds, it would suffer the regret $\sqrt{Kt}+\sqrt{t}$, and if it selects the optimal base algorithm, i.e.,~the base algorithm corresponding to the correct latent model, it would suffer the regret $\varepsilon t + \sqrt{t}$. Note that by regret, we mean the actual regret and not pseudo-regret, and thus, $\sqrt{t}$ is the consequence of noise in the reward signal. Thus, we select the reference regret bound of our regret balancing strategy as $U(t)=\min\{\varepsilon t+\sqrt{t},\sqrt{Kt}+\sqrt{t}\}$.
We may write the regret of the resulting regret balancing strategy as follows:
\begin{align*}
\mathcal R(T) &\stackrel{\text{(a)}}{=} \sum_{i=1}^{M+1} N_{i,T}\mu_* - R_{i,T} \stackrel{\text{(b)}}{\leq} \sum_{i=1}^{M+1} N_{i,T} b_T - R_{i,T} \stackrel{\text{(c)}}{=} \sum_{i=1}^{M+1} U(N_{i,T}) \\ 
&\leq \sum_{i=1}^{M+1} \min\big\{\varepsilon N_{i,T} + \sqrt{N_{i,T}}, \sqrt{K N_{i,T}} + \sqrt{N_{i,T}}\big\} \\ 
&\leq \min\Big\{\sum_{i=1}^{M+1}\big(\varepsilon N_{i,T} + \sqrt{N_{i,T}}\big), \sum_{i=1}^{M+1} \big(\sqrt{K N_{i,T}} + \sqrt{N_{i,T}}\big)\Big\} \\
&\stackrel{\text{(d)}}{=} \min\Big\{\varepsilon T + \sum_{i=1}^{M+1}\sqrt{N_{i,T}}, \sum_{i=1}^{M+1} \big(\sqrt{K N_{i,T}} + \sqrt{N_{i,T}}\big)\Big\} \\
&\stackrel{\text{(e)}}{\leq} \min\Big\{\varepsilon T + \sum_{i=1}^{M+1}\sqrt{\frac{T}{M+1}}, \sum_{i=1}^{M+1} \Big(\sqrt{K \frac{T}{M+1}} + \sqrt{\frac{T}{M+1}}\Big)\Big\} \\
&= \min\Big\{\varepsilon T + \sqrt{(M+1)T}, \sqrt{K(M+1)T} + \sqrt{(M+1)T}\Big\} \\ 
&= \mathcal O\left(\min\big\{\varepsilon T + \sqrt{MT}, \sqrt{K M T}\big\}\right), 
\end{align*}

which concludes our claim. Note that we used the following steps in our above derivations: {\bf (a)} $\mu_*$ is the mean of the optimal arm. {\bf (b)} This is because with high probability we have $\mu_*\leq b_t,\;\forall t\in[T]$. {\bf (c)} This is from the definition $b_t$ in~\eqref{eq:regret-balancing}. {\bf (d)} This is due to the fact that $\sum_{i=1}^{M+1}N_{i,T}=T$. {\bf (e)} The maximizer of $\sum_{i=1}^{M+1}\sqrt{N_{i,T}}$, subject to $\sum_{i=1}^{M+1}N_{i,T}=T$, is when all $\{N_{i,T}\}_{i=1}^{M+1}$ are equal.


\newpage
\section{ More on Experimental Results }

\begin{figure*}[t]
 \centering
     \includegraphics[width=\linewidth]{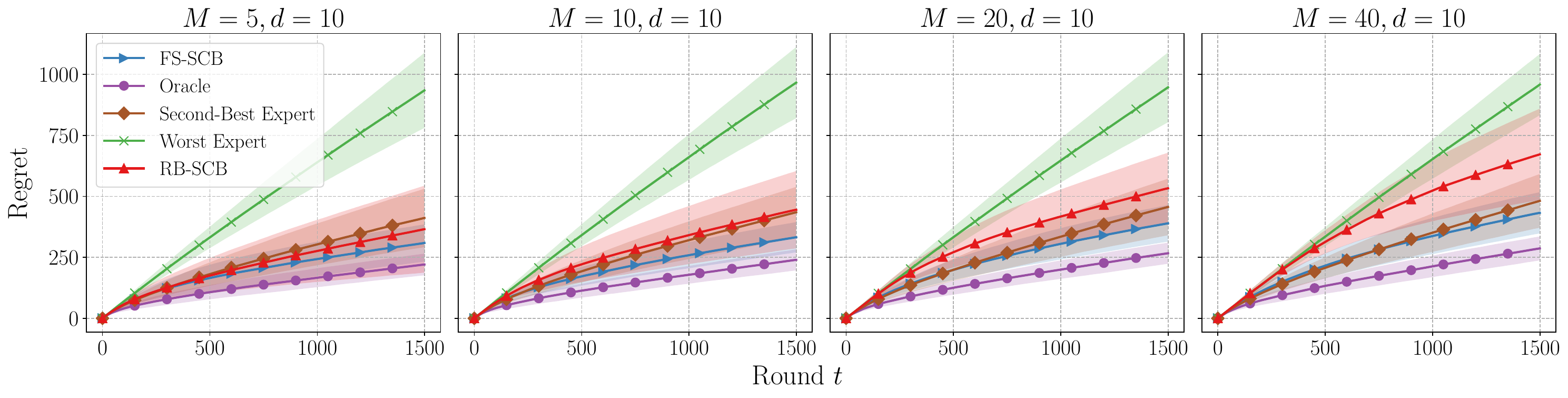}
     \vspace{-.3cm}
     \caption{Feature selection on MNIST dataset. The regrets are averaged over $100$ LB problems.}
     \label{fig:app_regret_fs-scb_mnist}
     \vspace{-.3cm}
\end{figure*}
\begin{figure*}[t]
 \centering
     \includegraphics[width=\linewidth]{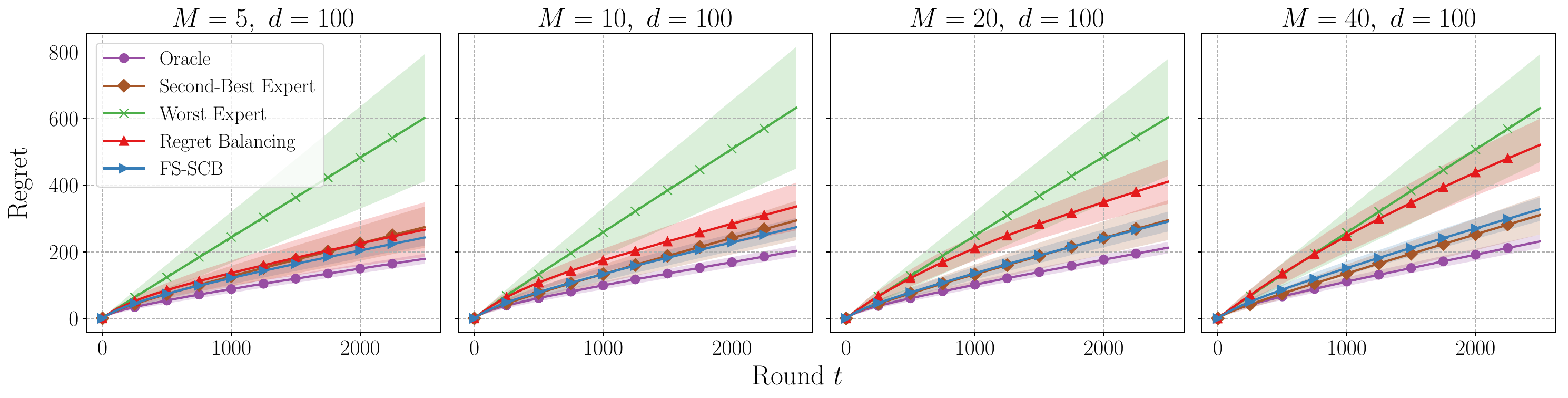}
     \vspace{-.3cm}
     \caption{Feature selection on CIFAR-100 dataset. The regrets are averaged over $100$ LB problems.}
     \label{fig:app_regret_fs-scb_cifar100}
     \vspace{-.3cm}
\end{figure*}
\label{app:experiment_details}

We evaluate the performances of FS-SCB and PS-OFUL algorithms in a synthetic linear bandit problem and real-world image classification problems on CIFAR-10, CIFAR-100 \citep{krizhevsky2009learning}, and MNIST datasets \citep{lecun1998gradient}.

\subsection{Feature Selection}

\subsubsection{MNIST Dataset}

MNIST dataset consists of 60000 training and 10000 test images of size $28\times28$, each belonging to one of 10 classes. We train a convolutional neural network (CNN) with $M$ different number of epochs on MNIST data, and use their second layer to the last as our $d=10$-dimensional feature maps $\{\phi^i\}_{i=1}^M$. These feature maps have test accuracy between $20\%$ (worst model) and $97\%$ (best model). We set the best one as true model $\phi^{i_*}$. For each class $s\in\mathcal S=\{0,\ldots,9\}$, we fit a linear model, given the feature map $\phi^{i_*}$, and obtain parameters $\{\theta^{i_*}_s\}_{s=0}^9$. At the beginning of each LB task, we select a class $s_*\in\mathcal S$ uniformly at random and set its parameter to $\theta^{i_*}_{s_*}$. At each round $t\in[T]$, the learner is given an action set consists of $10$ images, one from class $s_*$ and the rest randomly selected from the other classes. The reward of each action $a$ is defined as $\langle\phi^{i_*}(a),\theta^{i_*}_{s_*}\rangle+\eta_t\in[0,1]$, where $\phi^{i_*}(a)$ is the application of the feature map $\phi^{i_*}$ to the image corresponding to action $a$ and $\eta_t\sim\mathcal U[-0.5,0.5]$ is the noise. 

In Figure~\ref{fig:app_regret_fs-scb_mnist}, we compare the regret of our FS-SCB algorithm for different number of models $M$ with a regret balancing algorithm that uses SquareCB baselines (RB-SCB), and three SquareCB algorithms that use the best (Oracle), second-best (with test accuracy $84\%$), and worst feature maps (experts). Each plot is averaged over $100$ LB problems. Figure~\ref{fig:app_regret_fs-scb_mnist} shows that {\bf 1)} FS-SCB always performs between the best and second-best experts, {\bf 2)} the regret of FS-SCB that scales as $\sqrt{\log M}$ is close to RB-SCB (scales as $\sqrt{M}$) for small $M$, but gets much better as $M$ grows, and {\bf 3)} RB-SCB has much higher variance than the other algorithms.

\subsubsection{CIFAR-100 Dataset}
CIFAR-100 dataset consists of 50000 training and 10000 test images of size $32\times32$, each belonging to one of 100 classes. We extracted the features of the images by fine tuning and taking the output of the second-to-last layer of the EfficientNet-B0 Network \citep{tan2019efficientnet} and got the feature matrix of dimension $50000\times 1280$. For all experts $i\in[M]$, we multiply this feature matrix with a Gaussian random matrix of dimension $1280\times d_i$ for $d_i\in[2,128]$ to get the $d_i$ dimensional feature maps $\phi^i$. These feature maps have accuracy between $5\%$ (worst model) and $78\%$ (best model). We set the best one as true model $\phi^{i_*}$. For each class $s\in{\cal S}=\{0,\dots,99\}$, we fit a linear model, given the feature map $\phi^{i_*}$ and obtain parameters $\{\theta_s^{i_*}\}_{s=0}^{99}$. At the beginning of each LB task, we select a class $s_*\in\mathcal S$ uniformly at random and set its parameter to $\theta^{i_*}_{s_*}$. At each round $t\in[T]$, the learner is given an action set consists of $10$ images, one from class $s_*$ and the rest randomly selected from the other classes. The reward of each action $a$ is defined as $\langle\phi^{i_*}(a),\theta^{i_*}_{s_*}\rangle+\eta_t\in[0,1]$, where $\phi^{i_*}(a)$ is the application of the feature map $\phi^{i_*}$ to the image corresponding to action $a$ and $\eta_t\sim\mathcal U[-0.5,0.5]$ is the noise. 

In Figure~\ref{fig:app_regret_fs-scb_cifar100}, we compare the regret of our FS-SCB algorithm for different number of models $M$ with a regret balancing algorithm that uses SquareCB baselines (RB-SCB) and aggregate them according to~\eqref{eq:regret-balancing}, and three SquareCB algorithms that use the best (Oracle), second-best (with test accuracy $55\%$), and worst feature maps (experts). Each plot is averaged over $100$ LB problems. Figure~\ref{fig:app_regret_fs-scb_cifar100} shows that {\bf 1)} FS-SCB always performs close to the best and second-best experts, {\bf 2)} the regret of FS-SCB that scales as $\sqrt{\log M}$ is close to RB-SCB (scales as $\sqrt{M}$) for small $M$, but gets much better as $M$ grows, and {\bf 3)} RB-SCB has much higher variance than the other algorithms.

\subsection{Parameter Selection}

\subsubsection{Image Classification on MNIST Dataset}

\begin{figure}[t]
    \centering
    \includegraphics[width=.75\linewidth]{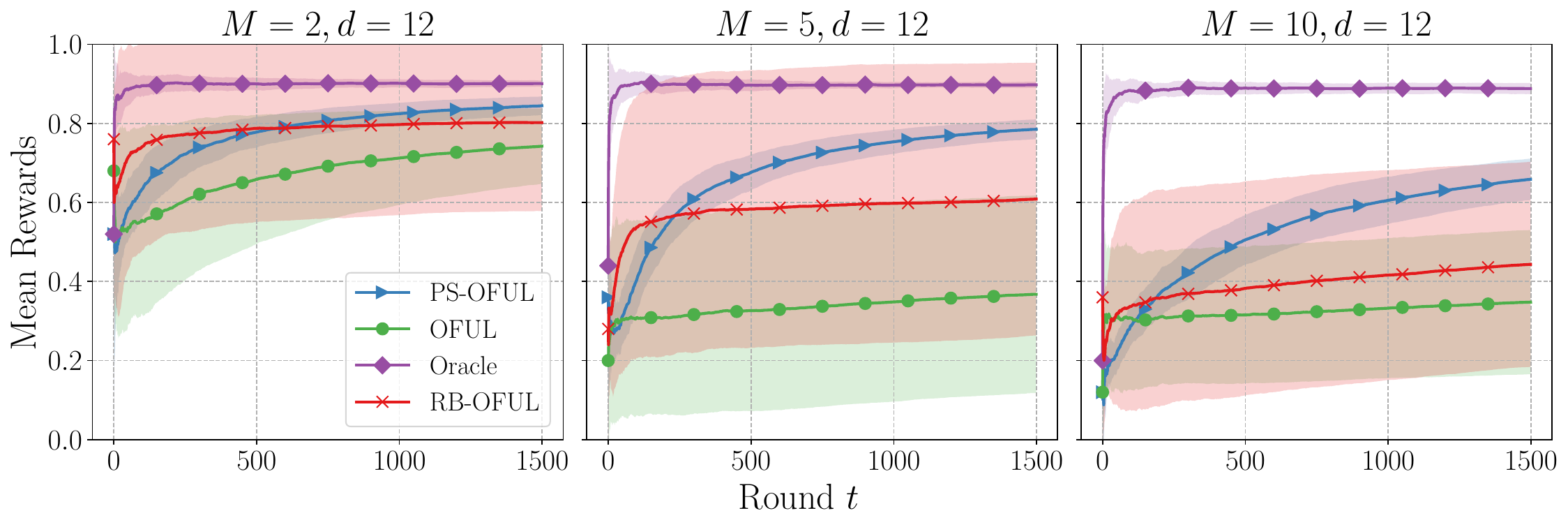}
    \caption{Parameter selection on MNIST dataset, where $100M$ datasets of size 500 are used to define the balls. The results are averaged over 50 runs.}
    \label{fig:app_mnist_ps}
\end{figure}
\begin{figure}[t]
    \centering
    \includegraphics[width=.75\linewidth]{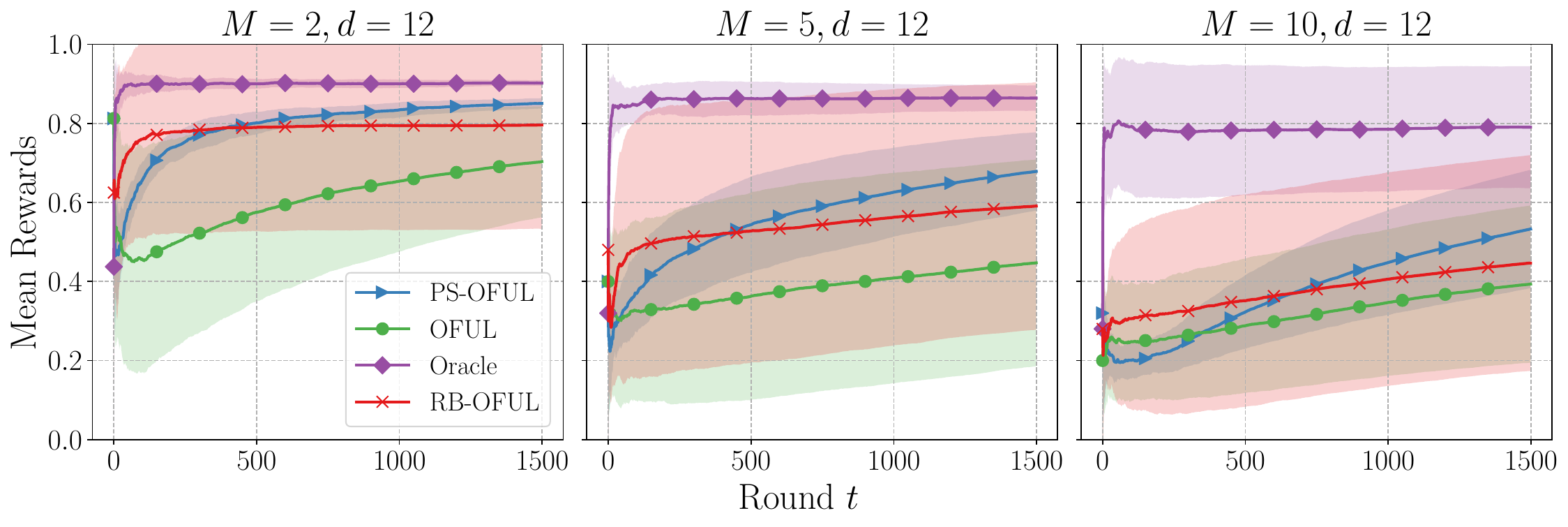}
    \caption{Parameter selection on MNIST dataset, where $10M$ datasets of size 50 are used to define the balls. The results are averaged over 50 runs.}
    \label{fig:app_mnist_ps_largervar}
\end{figure}

MNIST dataset consists of 60000 test and 10000 training images of size $28\times28$, each belonging to one of 10 classes. We train a CNN
with $d=12$ neurons on second-to-last layer on MNIST dataset with $98\%$ accuracy.  We then select this $d$-dimensional layer as our feature map $\phi$. To define our $M$ models (balls), we sample $100M$ datasets of size $500$. For each dataset, we randomly select a class $s_*\in[M]$, assign reward $1$ to images from $s_*$ and $0$ to other images, and fit a linear model to it to obtain a parameter vector. Finally, we fit a Gaussian mixture model (GMM) with $M$ components to these $100M$ parameter vectors and use the means and covariances of the resulting clusters as the center and radii of our $M$ models (balls). At the beginning of each LB task, we select a class $s_*\in[M]$ uniformly at random. At each round $t\in[T]$, the learner is given an action set consists of $10$ images, one from class $s_*$ and the rest randomly selected from the other classes. The learner receives a reward from Ber$(0.9)$ if it selects the image from class $s_*$, and from Ber$(0.1)$, otherwise.  

In Figure~\ref{fig:app_mnist_ps}, we compare the mean reward of PS-OFUL for different number of models $M$ with a regret balancing algorithm that uses OFUL baselines (RB-OFUL)~\citep{APP-2020}, OFUL (individual learning), and BIAS-OFUL~\citep{cella2020meta} with bias being the center of the true model (Oracle). Figure~\ref{fig:app_mnist_ps} shows {\bf 1)} the good performance of PS-OFUL, {\bf 2)} the performance of PS-OFUL gets better than RB-OFUL as $M$ grows ($\sqrt{\log M}$ vs.~$\sqrt{M}$ scaling), {\bf 3)} the large variance of RB-OFUL, especially in comparison to PS-OFUL, and finally {\bf 4)} the advantage of transfer (PS-OFUL) over individual (OFUL) learning. 

\paragraph{Impact of the model estimates:} In order to demonstrate the impact of the accuracy of the model center estimates as well as the radii of the balls, we defined a less accurate set of $M$ models (balls) using $10M$ datasets of size 50 (as opposed to $100M$ datasets of size 500). In Figure~\ref{fig:app_mnist_ps_largervar}, we compare the mean reward of PS-OFUL for different number of models $M$ with RB-OFUL, OFUL, and BIAS-OFUL.

\end{document}